\documentclass{article}

\usepackage[preprint]{neurips_2026}

\usepackage[utf8]{inputenc} 
\usepackage[T1]{fontenc}    
\usepackage{hyperref}       
\usepackage{url}            
\usepackage{booktabs}       
\usepackage{amsfonts}       
\usepackage{nicefrac}       
\usepackage{microtype}      
\usepackage{xcolor}         
\usepackage{multirow}
\usepackage{graphicx}
\usepackage{caption}
\usepackage{amsmath, amssymb, amsthm}
\usepackage{mathtools}
\usepackage{bm}
\usepackage{cleveref}

\usepackage{graphicx}
\usepackage{booktabs}
\usepackage{caption}
\usepackage{subcaption}

\usepackage[ruled,vlined]{algorithm2e}
\newcommand{\RR}{\mathbb{R}}
\newcommand{\EE}{\mathbb{E}}
\newcommand{\cL}{\mathcal{L}}
\newcommand{\cN}{\mathcal{N}}
\newcommand{\cS}{\mathcal{S}}
\DeclareMathOperator{\Span}{span}
\newcommand{\rank}{\operatorname{rank}}
\newcommand{\eqdef}{\triangleq}
\newcommand{\old}{\mathrm{old}}
\newcommand{\dir}{\mathrm{dir}}

\theoremstyle{plain}
\newtheorem{theorem}{Theorem}[section]
\newtheorem{lemma}[theorem]{Lemma}

\theoremstyle{definition}
\newtheorem{definition}[theorem]{Definition}
\newtheorem{assumption}[theorem]{Assumption}

\title{PaLoRA: Paced Low-Rank Adaptation for Continual Learning}

\author{%
  Yuxuan Li \\
  School of Computer Science \\
  Peking University \\
  \texttt{lyxzcx@outlook.com}
  \And
  Fanhu Zeng \\
  Institute of Automation \\
  Chinese Academy of Sciences \\
  \texttt{challengezengfh@gmail.com}
  \And
  Hao Tang\thanks{Corresponding author.} \\
  School of Computer Science \\
  Peking University \\
  \texttt{bjdxtanghao@gmail.com}
}

\begin{document}

\maketitle

\begin{abstract}
LoRA-based continual learning methods mitigate catastrophic forgetting through various mechanisms, yet nearly all complement these with \textbf{small learning rates as a heuristic to restrict gradient scaling magnitude}. Such fixed heuristics lack theoretical guidance on how the strength of this restriction should evolve as tasks accumulate. We reveal that even under directional constraints such as nullspace projection, finite-precision updates inevitably leak into the subspace of accumulated prior knowledge along multiple directions. While small learning rates attenuate such leakage, they cannot prevent the accumulated forgetting from intensifying as the effective rank of historical knowledge grows. We show that the optimal magnitude restriction should adaptively increase with this effective rank to balance \textbf{stability and plasticity}, \textit{i.e.}, preservation of previous knowledge and acquisition of new task information. Under an anisotropic leakage model, we derive a pacing law $s^*=\sqrt{R/c}$ that characterizes the optimal \textbf{scaling of gradient steps,} \textit{i.e.}, the magnitude restriction itself, where $R$ is the effective rank of past updates. Based on this insight, we propose PaLoRA, which compresses historical knowledge via adaptive SVD truncation, projects gradients onto the nullspace of prior tasks, and applies rank-aware adaptive pacing. Experiments demonstrate consistent improvements over prior methods, with particularly strong performance in long-horizon settings, achieving substantial gains of 4\% accuracy on challenging 50-task ImageNet-A and ImageNet-R benchmarks. The code is available at: \href{https://github.com/liyuxuan-github/PaLoRA}{https://github.com/liyuxuan-github/PaLoRA}
\end{abstract}

\section{Introduction}
Continual learning~\cite{parisi2019continual,masana2022class,van2022three,zhou2024expandable,guo2025mcitlib}, which refers to the ability to learn sequentially from a stream of tasks without catastrophic forgetting~\cite{robins1995catastrophic,kemker2018measuring,guo2026cl}, remains a fundamental challenge in machine learning. With the rise of large pre-trained vision transformers (ViTs)~\cite{dosovitskiy2020image,khan2022transformers,han2022survey} as standard backbones for visual recognition~\cite{touvron2021training,steiner2021train,DBLP:journals/ijcv/KongXLDTWM25,DBLP:conf/ics/LiLXKS0XDDWLF24,DBLP:conf/nips/DongLWLYTW23}, parameter-efficient fine-tuning methods~\cite{li2021prefix,hu2022lora,gao2024parameter}, particularly Low-Rank Adaptation (LoRA)~\cite{hu2022lora}, have become a prominent paradigm for continual adaptation. By learning compact, task-specific low-rank updates while keeping pre-trained weights frozen, LoRA-based continual learning methods~\cite{liang2024inflora,liu2025lora} achieve an appealing balance between computational efficiency and representational capacity.

Existing LoRA-based continual learning approaches mitigate catastrophic forgetting primarily through directional constraints, \textit{e.g.}, restricting new task updates to the null space of previously learned knowledge ~\cite{wang2023orthogonal,liang2024inflora,zhu2025bilora}. While these methods achieve promising results, they focus on constraining the \emph{direction} of parameter updates while largely overlooking the \emph{magnitude}, \textit{i.e.}, how large the update step is. In practice, nearly all these approaches~\cite{wang2023orthogonal,liang2024inflora,zhu2025bilora,liu2025lora} further adopt small learning rates \cite{zhang2023slca}, which equivalently restrict the magnitude of updates, as an additional safeguard to enhance stability. However, \textbf{(1) why restricting update magnitude mitigates forgetting} remains theoretically unexplored. More importantly, as tasks accumulate, the volume and complexity of historical knowledge increase, rendering a fixed restriction strength increasingly suboptimal; \textbf{(2) therefore, how such restriction should adapt as tasks accumulate} has received no principled investigation.

We illustrate the optimization trajectories in Figure~\ref{fig:insight} to explain why restricting update magnitude mitigates forgetting. Without constraints, fine-tuning tends to move the model toward the task-specific optimum of the new task, which may deviate from the joint optimum over all tasks. Restricting updates instead limits excessive parameter shifts and keeps optimization closer to the joint solution, achieving a better balance between plasticity for learning new knowledge and stability for preserving previous knowledge. As tasks accumulate in Figure~\ref{fig:insight}b, the current parameters encode increasing historical knowledge, reflected by the effective rank of accumulated updates, while the optimum of a single new task becomes less representative of the overall objective. This progressively shifts the trade-off toward stability, suggesting that the restriction strength should be adaptively increased with the effective rank rather than kept fixed.

\begin{figure}[t]
    \centering
    \includegraphics[width=0.95\linewidth]{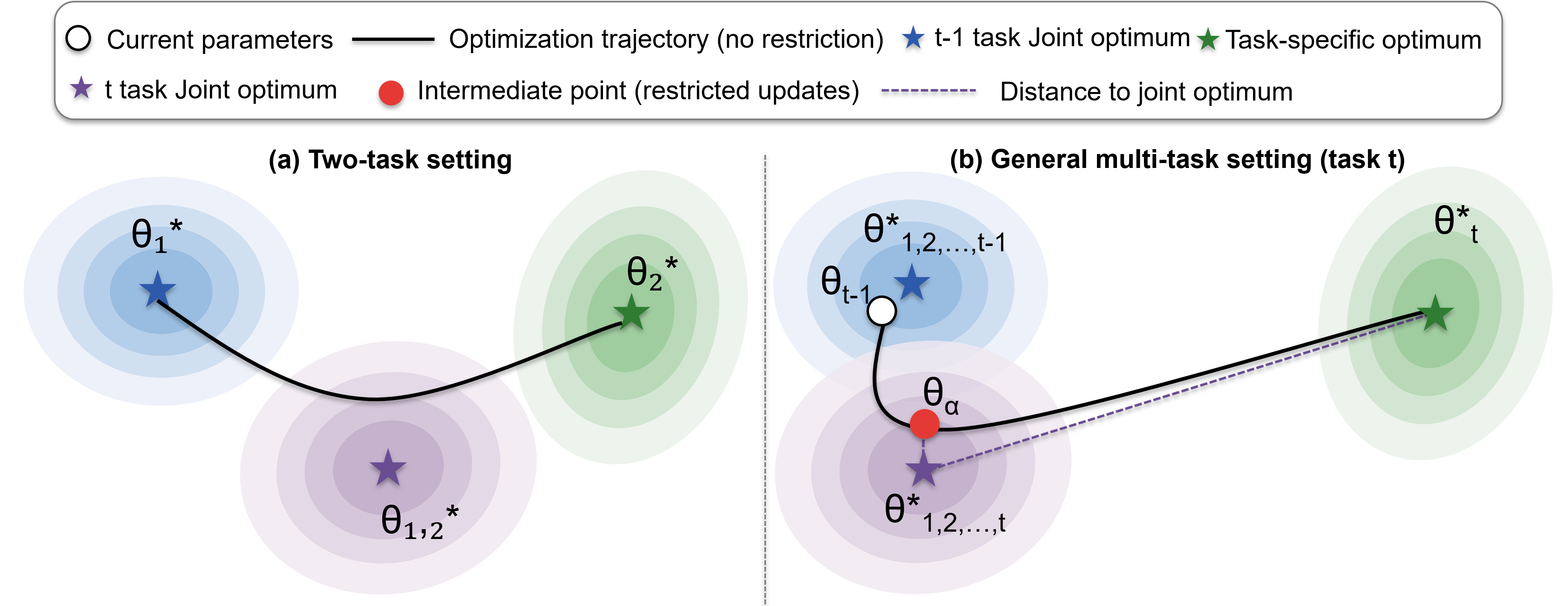}
    \caption{\small Adaptive parameter update restrictions progressively shift optimization trajectory toward previous task solutions as more tasks are encountered, better balancing stability and plasticity to approach the joint optimum.}
    \vspace{-6pt}
    \label{fig:insight}
\end{figure}

To address this, we analyze rank-aware scaling, i.e., magnitude restriction for LoRA updates, in a single-layer linear model under anisotropic leakage and local quadratic old-task sensitivity. We derive a \emph{pacing law} $s^*=\sqrt{R/c}$ for the \textbf{gradient-step denominator}, where $R$ is the effective rank of accumulated updates and $c$ sets the per-step forgetting budget relative to the average per-direction leakage floor. Under this model, $s^*$ maximizes first-order new-task progress subject to the budget: weaker restriction violates the budget, while stronger restriction sacrifices progress beyond what the budget requires. For deep ViTs, this result motivates a scaling principle rather than a formal guarantee.

Building upon this foundation, we propose Paced LoRA (PaLoRA), which compresses accumulated updates via adaptive SVD truncation to obtain effective rank $R$, projects new-task gradients onto the historical nullspace to reduce interference, and divides the gradient of the LoRA factor $B_t$ by $\sqrt{R/C}$ to instantiate rank-aware pacing. An empirically motivated feature-space regularizer complements these constraints by promoting current-task discrimination; the pacing-law optimality claim does not extend to this full regularized framework.

We evaluate PaLoRA on CIFAR-100, ImageNet-R, and ImageNet-A under different task splits. In the five-seed main comparisons, PaLoRA achieves the highest final and average accuracy across all splits. At 50 tasks, its final accuracy exceeds the strongest baseline in these tables, LoRA$^{-}$DRS, by 4.09 percentage points on ImageNet-R and 3.99 points on ImageNet-A. From 10 to 50 tasks, PaLoRA's final accuracy decreases by only 2.19 and 0.65 points on these datasets, respectively.

Our contributions are threefold: (1) under the stated leakage and local sensitivity assumptions, we establish how update restriction should grow with historical effective rank; (2) we derive the $\sqrt{R/c}$ pacing law for the budget-constrained model and instantiate this principle in PaLoRA through adaptive rank compression, nullspace projection, and rank-aware gradient scaling; and (3) we demonstrate strong performance on three benchmarks, including long-horizon (50-task) continual learning, with five-seed main comparisons.

\section{Related Work and Preliminaries}
\subsection{Related Work}
\textbf{Parameter-Efficient Fine-Tuning} (PEFT) aims to adapt pre-trained models to downstream tasks with minimal trainable parameters while keeping the pre-trained weights frozen~\cite{hu2022lora, jia2022visual}. Recent advances in PEFT have led to several distinct categories. For instance, selective methods~\cite{lee2019would,xu2021raise} fine-tune only a subset of parameters; additive methods~\cite{pfeiffer2021adapterfusion,sung2022lst,jie2024convolutional} insert lightweight adapter modules into intermediate layers; prompt-based methods~\cite{li2021prefix,zeng2025local,liu2024gpt} optimize soft prompt vectors; and reparameterization-based methods~\cite{zhang2023adalora,kopiczko2023vera,liu2024dora} sparsify weight matrices to reduce the number of trainable parameters. Among these, Low-Rank Adaptation (LoRA)~\cite{hu2022lora} substantially reduces trainable parameters through low-rank decomposition of weight matrices and incurs no additional inference overhead, making it widely adopted~\cite{zeng2025robustmerge,DBLP:conf/aaai/ShaoLLCYLYMTG26,DBLP:journals/corr/abs-2608-22370} and continuously developed~\cite{hayou2024lora+,kalajdzievski2023rank}.

\textbf{Continual Learning} seeks to address catastrophic forgetting when models learn from sequentially arriving data streams~\cite{guo2025continual}. Early work~\cite{rebuffi2017icarl,castro2018end,aljundi2019gradient,chrysakis2020online} primarily focused on retaining a small subset of historical data for rehearsal to prevent the model from forgetting past knowledge. However, due to storage constraints and privacy concerns, example-free methods have also attracted significant attention. For instance, LwF~\cite{dhar2019learning} constrains the model's outputs to remain close to those of the historical model, while Gradient Projection Memory (GPM)~\cite{saha2021gradient} reduces interference with prior knowledge via gradient orthogonalization.

With the advancement of PEFT, leveraging pre-trained weights for continual learning has gained increasing popularity. For example, L2P~\cite{wang2022learning} and its subsequent~\cite{wang2022dualprompt,smith2023coda,zeng2025modalprompt,DBLP:journals/corr/abs-2503-01887} variants employ multiple trainable prompts to preserve historical knowledge, thereby avoiding explicit storage of past data. Meanwhile, by harnessing knowledge from pre-trained weights, these methods better distinguish between different tasks with promising results. Recently, LoRA-based continual learning methods have been extensively studied~\cite{guo2025hide,guo2025federated,liang2024inflora,DBLP:conf/iclr/WuPHW0PMM025}. For instance, O-LoRA~\cite{wang2023orthogonal} constrains the orthogonality of parameters across different tasks through regularization; BiLoRA~\cite{zhu2025bilora}, based on FourierFT~\cite{gao2024parameter}, assigns nearly disjoint entries to different tasks to minimize mutual interference. Nevertheless, these methods only consider isolation of learning directions across tasks, and in practice, gradient leakage along these directions still causes forgetting to increase as the number of tasks grows.

\subsection{Preliminaries}
\textbf{Problem Definition.} We consider \emph{example-free class-incremental learning} (EFCIL), where a model is trained on a sequence of tasks $\{\mathcal{T}_t\}_{t=1}^T$ without access to past data. Each task $\mathcal{T}_t$ introduces a disjoint set of classes $\mathcal{Y}_t$, and we denote the cumulative label space up to step $t$ as $\mathcal{Y}_{1:t} = \bigcup_{\tau=1}^t \mathcal{Y}_\tau$.

At training time $t$, the learner has access only to samples from the current distribution $\mathcal{D}_t$ over $\mathcal{X} \times \mathcal{Y}_t$, and no data from previous tasks. At inference time, the model must predict over the entire label space $\mathcal{Y}_{1:T}$ \emph{without access to task identity}.

Let $f(x; W) \in \mathbb{R}^{|\mathcal{Y}_{1:T}|}$ denote the model. The objective at step $t$ is to minimize:
\[
\mathcal{L}_t(W) = \mathbb{E}_{(x,y)\sim \mathcal{D}_t}[\ell(f(x; W), y)],
\]
while maintaining performance on previously learned classes.

\textbf{Low-Rank Adaptation.} We consider low-rank adaptation (LoRA)~\cite{hu2022lora} for parameter-efficient fine-tuning of pre-trained models. Given a frozen pre-trained weight matrix $W_0 \in \mathbb{R}^{d \times d}$, LoRA models task-specific adaptation by adding a low-rank update:
\[
W = W_0 + BA,
\]
where $A \in \mathbb{R}^{r \times d}$ and $B \in \mathbb{R}^{d \times r}$ with rank $r \ll d$. This parameterization restricts learning to a low-dimensional subspace while keeping the original model weights fixed.
\section{Methodology}
\subsection{On the Role of Restricted Updates in Continual Learning}
Recently, LoRA-based methods~\cite{liang2024inflora,zhu2025bilora,liu2025lora} often follow the setting of SLCA~\cite{zhang2023slca} that adopt a small learning rate during training. Similarly, related work~\cite{huang2025mind} has observed that reducing the number of training epochs yields comparable improvements in continual learning performance. However, there is limited understanding of why restricted updates can effectively mitigate forgetting.

In this section, we provide an intuitive explanation based on the perspective of model update trajectories. Specifically, we argue that restricting the magnitude of parameter movement helps reduce interference with previously learned representations. Furthermore, this restriction should be adaptively strengthened as the number of tasks increases, since accumulated updates progressively increase the risk of cross-task interference.

We first illustrate our key insight using a two-task setting in Figure~\ref{fig:insight}a. Let $\theta^*_1$ and $\theta^*_2$ denote the optima obtained by training on task 1 and task 2 separately, and $\theta^*_{1,2}$ the joint optimum from simultaneous training. After convergence on task 1, unrestricted training on task 2 leads to $\theta^*_2$, which can be far from $\theta^*_{1,2}$. By restricting parameter updates, the model is constrained to an intermediate point along the trajectory between $\theta^*_1$ and $\theta^*_2$, potentially achieving a better balance between preserving previous knowledge and acquiring new knowledge.

This principle extends to the general $t$-task setting in Figure~\ref{fig:insight}b. After learning $t-1$ tasks, the model converges to $\theta_{t-1}$. Unrestricted training on task $t$ leads to the task-specific optimum $\theta^*_t$, deviating from the joint optimum $\theta^*_{1,2,\ldots,t}$. Restricted updates constrain the trajectory, yielding an intermediate point $\theta_\alpha$. We denote by $\theta_{\alpha^*}$ the point on this constrained trajectory closest to the true joint optimum. As the number of tasks increases, $\theta_{t-1}$ becomes increasingly representative of performance on all previous tasks, while $\theta^*_t$ reflects only the current task. Consequently, the optimal trade-off point $\theta_{\alpha^*}$ shifts closer to $\theta_{t-1}$.

This insight suggests that the strength of update restriction should not remain fixed, as is commonly assumed in prior work using small learning rates. Instead, it should be adaptively increased as the number of tasks grows, to better balance stability on previous tasks and plasticity for new tasks.

\subsection{Theoretical Analysis: Rank-Aware Scaling for LoRA-Based Continual Learning}
\label{sec:rank-aware-theory}

\textbf{Setup.} 
Consider a single linear layer with frozen pre-trained weights $W_0\in\RR^{d\times d}$. 
After learning tasks $1,\dots,t-1$ via LoRA, the accumulated low-rank updates form an effective adaptation $C_{t-1}\in\RR^{d\times d}$ with $\rank(C_{t-1})=R$. 
The current effective weight is $W^{(t-1)}=W_0+C_{t-1}$. 
Task $t$ introduces an additional LoRA increment $\Delta_t$, and the updated weight becomes $W^{(t)}=W^{(t-1)}+\Delta_t$.
Each task $t$ is equipped with an input design matrix $X_t\in\RR^{d\times n_t}$ with positive-definite covariance $\Sigma_t=X_tX_t^\top\succ 0$. 
The task-$t$ objective is the Frobenius-norm regression loss
\[
\cL_t(\Delta_t)=\frac{1}{2}\|(W^{(t-1)}+\Delta_t)X_t-Y_t\|_F^2,
\]
optimized via full-batch gradient descent with fixed-step projected gradient descent.

\textbf{Nullspace Projection and Leakage.} 
To prevent forgetting, the new LoRA increment $\Delta_t$ is ideally restricted to the right nullspace of the accumulated adaptation $C_{t-1}$. 
Let $C_{t-1}=U\Sigma V^\top=\sum_{i=1}^{R}\sigma_i u_i v_i^\top$ be the compact SVD, and let $V_\perp=[v_{R+1},\dots,v_d]$ span the orthogonal complement of the row space of $C_{t-1}$. 
The feasible subspace for $\Delta_t$ is $\cN_t\triangleq\{M V_\perp^\top : M\in\RR^{d\times(d-R)}\}$, which ensures $\Delta_t x_{\old}=0$ for all old-task inputs $x_{\old}\in\Span(V_{[:R]})$ and preserves old-task forward passes under exact arithmetic (see Section~\ref{app:proofs-setup}).

Finite-precision arithmetic and imperfect projection motivate a leakage model in the old-task subspace $\cS=\Span\{u_i v_i^\top\}_{i=1}^R$.
Our derivation relies on two core assumptions: (i) anisotropic leakage with $\EE[\epsilon_i^2]=\eta^2\gamma_i^2/s^2$, where $\eta$ is the base learning rate, $\gamma_i$ the direction-dependent leakage scale, and $s>0$ the step-size denominator; and (ii) locally quadratic old-task loss with curvature $\lambda_i>0$ along each historical singular direction (Assumptions~\ref{ass:leakage}--\ref{ass:quadratic}).

\textbf{The Forgetting–Progress Trade-off.} 
Under these assumptions, the per-step expected old-task loss increase and first-order new-task loss change are:
\begin{align}
\EE[\Delta\cL_{t-1}] &= \frac{\eta^2 R\Gamma^2}{2s^2},
\qquad \Gamma^2\eqdef\frac1R\sum_{i=1}^R\lambda_i\gamma_i^2,
\label{eq:expected-forgetting}\\
\Delta\cL_t &\approx -\frac{\eta G^2}{s},
\qquad\;\; G^2=\|P_{\cN_t}(\nabla_{\Delta_t}\cL_t)\|_F^2.
\label{eq:new-task-progress}
\end{align}
Here, $\Gamma^2$ averages the leakage--curvature products, $P_{\cN_t}(M)=MV_\perp V_\perp^\top$ is the orthogonal projection onto $\cN_t$, and $G^2$ is the squared norm of the projected gradient; new-task progress is $-\Delta\cL_t$.
To see the trade-off, local quadratic sensitivity gives $\EE[\Delta\cL_{t-1}]\approx\frac12\sum_i\lambda_i\EE[\epsilon_i^2]$, yielding Eqn.~\eqref{eq:expected-forgetting}, while a first-order Taylor expansion along the projected step gives Eqn.~\eqref{eq:new-task-progress}. Thus, increasing $s$ suppresses forgetting quadratically but reduces progress only linearly. Full proofs are provided in Section~\ref{app:proofs-lemmas}.

\begin{figure}[t]
    \centering
    \includegraphics[width=1.0\linewidth]{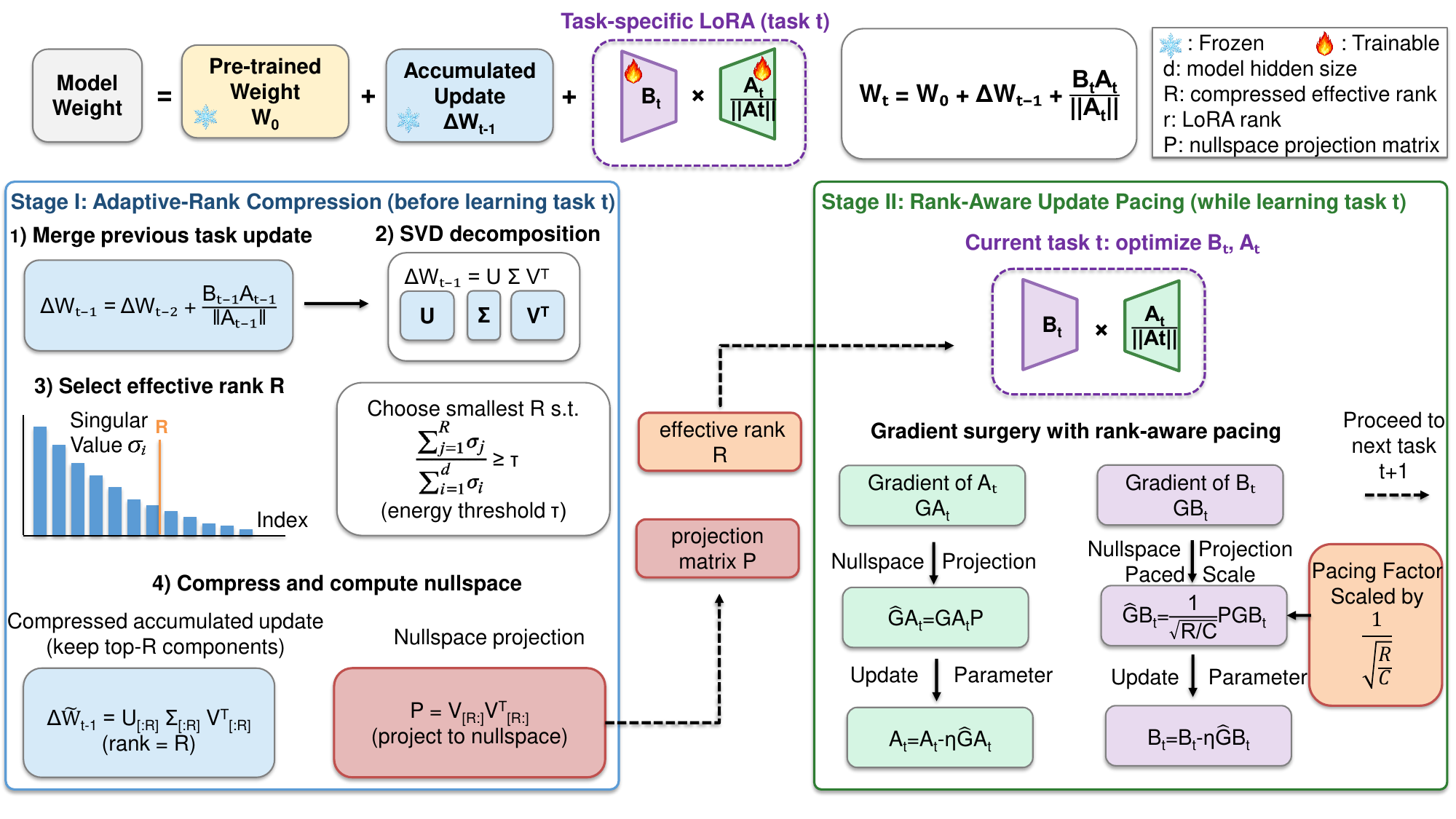}
    \caption{\small \textbf{Overview of PaLoRA.} (A) Model decomposition: pretrained weights and compressed memory from previous tasks are frozen, while a new low-rank update is learned for the current task. (B) Two-stage process: adaptive-rank compression of accumulated knowledge via SVD, followed by rank-aware gradient pacing with nullspace projection and scaled updates.}
    \vspace{-6pt}
    \label{fig:main-structure}
\end{figure}

\textbf{Forgetting Budget.}
We maximize new-task progress subject to a per-step expected forgetting cap $\epsilon_0=c\,\epsilon_{\dir}$, where $\epsilon_{\dir}=\eta^2\Gamma^2/2$ is the average per-direction forgetting from an unscaled step and $c>0$ is a tolerance coefficient. The optimum is the smallest $s$ satisfying this cap.

\begin{lemma}[The $\sqrt{R/c}$ Scaling Law]
\label{lem:sqrt-law}
Under the stated model, with budget $\epsilon_0=c\,\eta^2\Gamma^2/2$, the optimal scale simplifies to:
\begin{equation}
\label{eq:sqrt-law}
\,s^*=\sqrt{R/c}\,.
\end{equation}
\end{lemma}

\textbf{Properties.} 
Substituting $s^*=\sqrt{R/c}$ into Eqn.~\eqref{eq:expected-forgetting} shows that the per-step forgetting exactly equals the prescribed budget $c\,\epsilon_{\dir}$, regardless of the specific directions in the rank-$R$ subspace. Any $s<\sqrt{R/c}$ violates the budget; any $s>\sqrt{R/c}$ is suboptimal by Eqn.~\eqref{eq:new-task-progress}. Hence $s^*$ is the unique Pareto-optimal operating point for this budget-constrained pacing problem under the stated model. Full derivations are in Section~\ref{app:proofs-properties}.

\subsection{PaLoRA Framework}

Motivated by the scaling law $s^*=\sqrt{R/c}$ under the stated model (Lemma~\ref{lem:sqrt-law}), we now present the practical PaLoRA framework, illustrated in Figure~\ref{fig:main-structure}. Our design translates the rank-aware scaling principle into a two-stage algorithm: adaptive-rank compression of accumulated knowledge, followed by paced gradient updates governed by the effective rank $R$. Following LoRA-style parameterization, we learn a sequence of low-rank factors $\{A_i, B_i\}$ with $\rank$ r to encode task-specific knowledge. During training on task $t$, the pretrained weights $W_0$ and the accumulated update $\Delta W_{t-1}$ are frozen, and only $(A_t, B_t)$ are optimized.

A key distinction of our method is that we normalize $A_t$, enforcing it to capture only directional changes, while $B_t$ controls the magnitude of the update. Since $B_t$ is initialized to zero, scaling its gradient provides a direct mechanism to regulate the learning dynamics of the entire task-specific update.

\textbf{Stage I: Adaptive-Rank Compression.} Before learning task $t$, we merge task t-1 updates $\frac{B_{t-1}A_{t-1}}{\|A_{t-1}\|}$ into the accumulated matrix $\Delta W_{t-2}$ and perform SVD decomposition. While assuming a fixed rank growth of $(t-1)r$ is often suboptimal due to redundancy knowledge across tasks and varying task difficulty.

To address this, we define an \emph{effective rank} $R$ by retaining the top singular values whose cumulative energy exceeds a predefined threshold $\tau$. We then reconstruct a compressed representation of $\Delta W_{t-1}$ using these $R$ components and compute its associated nullspace projection matrix $P$. This procedure adaptively removes redundant directions and yields a compact representation of past knowledge.

\textbf{Stage II: Rank-Aware Update Pacing.} During training on task $t$, we use the effective rank $R$ and projection matrix $P$ to constrain parameter updates. The gradient of $A_t$ is projected onto the nullspace of $\Delta W_{t-1}$ to reduce interference with previous tasks. For $B_t$, we additionally apply gradient scaling by $\sqrt{R/C}$, where $C$ is a budget constant. This induces a rank-dependent pacing mechanism that adaptively controls the trade-off between stability and plasticity.

\paragraph{Plasticity Compensation via Regularization.}
While nullspace projection reduces interference with
previously learned knowledge, it may also restrict the
model's ability to learn new tasks. To compensate for
this loss of plasticity, we adopt a regularization term
inspired by the easy-positive triplet loss~\cite{xuan2020improved}.

Specifically, for each anchor, we select the closest
positive and negative samples to construct a
margin-based objective:
\begin{equation}
\mathcal{L}_{\mathrm{reg}}
=
\frac{1}{N}\sum_{i=1}^{N}
\left[
d_i^{+}-d_i^{-}+m
\right]_{+},
\label{eq:reg}
\end{equation}
where $d_i^{+}$ and $d_i^{-}$ denote the minimum
Euclidean distances to samples of the same and
different classes, respectively.
In our implementation, the positive set includes
the anchor itself, making $d_i^{+}$ exactly
zero. Consequently, the regularizer primarily
enforces a margin against the nearest negative
sample, encouraging inter-class separation under
constrained parameter updates.
The training objective $\mathcal{L}_{\text{total}}$ combines the cross-entropy loss $\mathcal{L}_{\text{CE}}$ and the regularization term, balanced by a weighting factor $\lambda$:
\[
\mathcal{L}_{\text{total}} = \mathcal{L}_{\text{CE}} + \lambda \mathcal{L}_{\text{reg}}.
\]

\textbf{Connection to Theory.}
Nullspace projection defines the feasible update subspace, adaptive SVD provides a theory-motivated estimate of the historical effective rank $R$, and pacing instantiates the derived scaling, with $C$ serving as the practical budget coefficient.
The triplet loss is an empirically motivated feature-space plasticity compensator; it is not derived from the pacing law, whose Pareto-optimality claim does not cover the full regularized system.
For deep ViTs, the analysis motivates an approximate local scaling principle for LoRA updates rather than a formal guarantee for the full nonlinear network.

\section{Experiments}
\subsection{Experimental Setup}

\textbf{Datasets and Evaluation Metrics.} We evaluate our method on three benchmark datasets: CIFAR-100~\cite{krizhevsky2009learning}, ImageNet-R~\cite{hendrycks2021many}, and ImageNet-A~\cite{hendrycks2021natural}. Following prior work~\cite{liu2025lora}, we split each dataset into task sequences with 10, 20, 25, and 50 tasks to evaluate performance under different continual learning scenarios.

Following previous studies~\cite{liang2024inflora,zhu2025bilora,liu2025lora}, we adopt $ACC_{T}$ and $\overline{ACC_{T}}$ as evaluation metrics. $ACC_{T}$ denotes the average accuracy over all tasks after learning the final task $T$, while $\overline{ACC_{T}}$ represents the average of $ACC_i$ computed after each task $i$ is learned. These metrics jointly reflect the final performance and the stability of the model throughout the continual learning process. Following prior work~\cite{liu2025lora}, we further report Backward Transfer (BWT) to quantify forgetting. BWT is computed as the average difference between the accuracy of each task after learning the final task $T$ and its accuracy immediately after it is learned, thereby measuring how much previously acquired knowledge is retained during continual learning.

The main accuracy tables report means and standard deviations over five random seeds (0--4). Other tables and the existing accuracy curves report point estimates. Implementation details and additional analyses are provided in Appendix~\ref{app:more-experiments}.

\textbf{Compared Methods.} We compare our method with several LoRA-based continual learning approaches, including InfLoRA~\cite{liang2024inflora}, BiLoRA~\cite{zhu2025bilora}, and LoRA$^{-}$DRS~\cite{liu2025lora}. In addition, we include prompt-based methods such as L2P~\cite{wang2022learning}, DualPrompt~\cite{wang2022dualprompt}, and CODA-Prompt~\cite{smith2023coda}. We additionally compare with SD-LoRA~\cite{DBLP:conf/iclr/WuPHW0PMM025} in Section~\ref{sec:sd-lora}.

We further introduce two additional baselines: \textbf{Pretrain} and \textbf{Sequence}. \textbf{Pretrain} directly uses a frozen pretrained model without any task adaptation, representing full stability but no plasticity. \textbf{Sequence} applies standard LoRA-based fine-tuning across tasks without any mechanism to mitigate forgetting, representing full plasticity but suffering from catastrophic forgetting.

\begin{table}[t]
  \caption{ImageNet-R accuracy (\%, mean $\pm$ standard deviation over five seeds). PaLoRA achieves the highest final and average accuracy across all splits, with a 4.09-point final-accuracy gain over LoRA$^{-}$DRS at 50 tasks.}
  \label{table:imgr}
  \centering
  \setlength{\tabcolsep}{3pt}
  \resizebox{1.0\linewidth}{!}{
  \begin{tabular}{lrrrrrrrr}
    \toprule
    Tasks & \multicolumn{2}{c}{10}& \multicolumn{2}{c}{20}& \multicolumn{2}{c}{25}& \multicolumn{2}{c}{50}\\
    \cmidrule(r){1-1}\cmidrule(r){2-3}\cmidrule(r){4-5}\cmidrule(r){6-7}\cmidrule(r){8-9}
    Method & $ACC_{10}$ & $\overline{ACC_{10}}$ &$ACC_{20}$& $\overline{ACC_{20}}$ &$ACC_{25}$& $\overline{ACC_{25}}$ &$ACC_{50}$& $\overline{ACC_{50}}$ \\
    \cmidrule(r){1-1}\cmidrule(r){2-3}\cmidrule(r){4-5}\cmidrule(r){6-7}\cmidrule(r){8-9}
    Pretrain&55.68$\pm$0.00&59.20$\pm$0.00&55.68$\pm$0.00&59.55$\pm$0.00&55.68$\pm$0.00&59.78$\pm$0.00&55.68$\pm$0.00&59.95$\pm$0.00\\
    Sequence&62.11$\pm$2.21&73.39$\pm$0.97&57.30$\pm$0.39&69.56$\pm$0.69&54.48$\pm$2.16&66.98$\pm$0.28&47.48$\pm$0.05&62.94$\pm$0.60\\
    \cmidrule(r){1-1}\cmidrule(r){2-3}\cmidrule(r){4-5}\cmidrule(r){6-7}\cmidrule(r){8-9}
    L2P~\cite{wang2022learning}&62.28$\pm$0.30&66.98$\pm$0.04&61.24$\pm$0.16&64.73$\pm$0.43&61.03$\pm$0.35&64.37$\pm$0.37&59.44$\pm$0.45&63.16$\pm$0.57\\
    DualPrompt~\cite{wang2022dualprompt}&58.91$\pm$0.23&62.66$\pm$0.06&57.78$\pm$0.33&61.91$\pm$0.14&57.53$\pm$0.64&61.66$\pm$0.10&56.28$\pm$0.97&61.72$\pm$0.22\\
    CODA-Prompt~\cite{smith2023coda}&69.81$\pm$0.06&73.00$\pm$0.15&66.42$\pm$0.13&69.55$\pm$0.35&65.03$\pm$0.23&68.22$\pm$0.18&56.19$\pm$0.16&59.01$\pm$0.27\\
    InfLoRA~\cite{liang2024inflora}&74.25$\pm$0.74&80.35$\pm$0.31&72.62$\pm$0.92&78.87$\pm$0.03&71.42$\pm$1.44&78.11$\pm$0.36&66.52$\pm$1.13&73.09$\pm$0.67\\
    BiLoRA~\cite{zhu2025bilora}&72.73$\pm$0.23&76.38$\pm$0.11&72.38$\pm$0.21&76.24$\pm$0.25&72.51$\pm$0.37&76.12$\pm$0.25&71.56$\pm$0.54&75.28$\pm$0.15\\
    LoRA$^{-}$DRS~\cite{liu2025lora}&75.56$\pm$0.23&81.42$\pm$0.06&75.12$\pm$0.08&80.67$\pm$0.45&74.68$\pm$0.07&79.93$\pm$0.24&72.47$\pm$0.56&77.19$\pm$0.64\\
    \cmidrule(r){1-1}\cmidrule(r){2-3}\cmidrule(r){4-5}\cmidrule(r){6-7}\cmidrule(r){8-9}
    \textbf{PaLoRA (Ours)}&\textbf{78.75$\pm$0.11}&\textbf{83.08$\pm$0.16}&\textbf{78.84$\pm$0.39}&\textbf{83.22$\pm$0.67}&\textbf{78.21$\pm$0.16}&\textbf{82.85$\pm$0.07}&\textbf{76.56$\pm$0.23}&\textbf{80.68$\pm$0.39}\\
    \bottomrule
  \end{tabular}}
\end{table}

\begin{table}[t]
  \caption{CIFAR-100 accuracy (\%, mean $\pm$ standard deviation over five seeds). PaLoRA consistently improves both metrics, reaching $88.96\pm0.14$ final accuracy at 50 tasks.}
  \label{table:cifar}
  \centering
  \setlength{\tabcolsep}{3pt}
  \resizebox{1.0\linewidth}{!}{
  \begin{tabular}{lrrrrrrrr}
    \toprule
    Tasks & \multicolumn{2}{c}{10}& \multicolumn{2}{c}{20}& \multicolumn{2}{c}{25}& \multicolumn{2}{c}{50}\\
    \cmidrule(r){1-1}\cmidrule(r){2-3}\cmidrule(r){4-5}\cmidrule(r){6-7}\cmidrule(r){8-9}
    Method & $ACC_{10}$ & $\overline{ACC_{10}}$ &$ACC_{20}$& $\overline{ACC_{20}}$ &$ACC_{25}$& $\overline{ACC_{25}}$ &$ACC_{50}$& $\overline{ACC_{50}}$ \\
    \cmidrule(r){1-1}\cmidrule(r){2-3}\cmidrule(r){4-5}\cmidrule(r){6-7}\cmidrule(r){8-9}
    Pretrain&83.81$\pm$0.00&88.82$\pm$0.00&83.81$\pm$0.00&89.13$\pm$0.00&83.81$\pm$0.00&89.16$\pm$0.00&83.81$\pm$0.00&89.33$\pm$0.00\\
    Sequence&87.54$\pm$0.94&92.12$\pm$0.78&81.51$\pm$2.60&89.05$\pm$0.48&81.93$\pm$4.13&88.95$\pm$1.11&77.52$\pm$5.94&85.01$\pm$4.23\\
    \cmidrule(r){1-1}\cmidrule(r){2-3}\cmidrule(r){4-5}\cmidrule(r){6-7}\cmidrule(r){8-9}
    L2P~\cite{wang2022learning}&84.53$\pm$0.71&87.84$\pm$0.39&81.00$\pm$0.51&84.70$\pm$0.94&78.45$\pm$0.37&82.08$\pm$0.08&74.48$\pm$1.35&78.90$\pm$1.78\\
    DualPrompt~\cite{wang2022dualprompt}&75.37$\pm$2.35&86.03$\pm$0.66&77.79$\pm$0.04&86.49$\pm$0.52&79.72$\pm$0.85&86.82$\pm$0.12&81.11$\pm$0.77&87.37$\pm$0.05\\
    CODA-Prompt~\cite{smith2023coda}&84.40$\pm$0.13&88.77$\pm$0.38&81.69$\pm$1.03&85.34$\pm$0.70&79.10$\pm$1.12&82.19$\pm$0.87&72.16$\pm$0.08&82.65$\pm$0.49\\
    InfLoRA~\cite{liang2024inflora}&86.51$\pm$1.13&91.51$\pm$0.30&82.40$\pm$5.13&89.17$\pm$1.62&82.32$\pm$2.04&88.62$\pm$1.48&76.04$\pm$1.92&83.31$\pm$2.78\\
    BiLoRA~\cite{zhu2025bilora}&87.01$\pm$0.61&90.69$\pm$0.50&84.91$\pm$0.70&88.96$\pm$0.85&85.32$\pm$1.16&89.66$\pm$1.83&79.99$\pm$3.82&85.51$\pm$3.04\\
    LoRA$^{-}$DRS~\cite{liu2025lora}&89.55$\pm$0.10&92.83$\pm$0.01&88.69$\pm$0.01&92.57$\pm$0.01&88.08$\pm$0.09&92.23$\pm$0.01&87.02$\pm$0.15&91.45$\pm$0.25\\
    \cmidrule(r){1-1}\cmidrule(r){2-3}\cmidrule(r){4-5}\cmidrule(r){6-7}\cmidrule(r){8-9}
    \textbf{PaLoRA (Ours)}&\textbf{91.02$\pm$0.05}&\textbf{94.00$\pm$0.01}&\textbf{90.26$\pm$0.07}&\textbf{93.64$\pm$0.04}&\textbf{90.40$\pm$0.21}&\textbf{93.68$\pm$0.02}&\textbf{88.96$\pm$0.14}&\textbf{92.51$\pm$0.02}\\
    \bottomrule
  \end{tabular}}
\end{table}

\begin{table}[t]
  \caption{ImageNet-A accuracy (\%, mean $\pm$ standard deviation over five seeds). PaLoRA leads across all splits and exceeds LoRA$^{-}$DRS by 3.99 points in final accuracy at 50 tasks.}
  \label{table:imga}
  \centering
  \setlength{\tabcolsep}{3pt}
  \resizebox{1.0\linewidth}{!}{
  \begin{tabular}{lrrrrrrrr}
    \toprule
    Tasks & \multicolumn{2}{c}{10}& \multicolumn{2}{c}{20}& \multicolumn{2}{c}{25}& \multicolumn{2}{c}{50}\\
    \cmidrule(r){1-1}\cmidrule(r){2-3}\cmidrule(r){4-5}\cmidrule(r){6-7}\cmidrule(r){8-9}
    Method & $ACC_{10}$ & $\overline{ACC_{10}}$ &$ACC_{20}$& $\overline{ACC_{20}}$ &$ACC_{25}$& $\overline{ACC_{25}}$ &$ACC_{50}$& $\overline{ACC_{50}}$ \\
    \cmidrule(r){1-1}\cmidrule(r){2-3}\cmidrule(r){4-5}\cmidrule(r){6-7}\cmidrule(r){8-9}
    Pretrain&48.85$\pm$0.00&57.39$\pm$0.00&48.85$\pm$0.00&58.47$\pm$0.00&48.85$\pm$0.00&58.56$\pm$0.00&48.85$\pm$0.00&58.97$\pm$0.00\\
    Sequence&34.13$\pm$1.44&53.86$\pm$2.04&23.87$\pm$0.23&46.07$\pm$1.37&25.97$\pm$3.78&43.62$\pm$0.24&15.61$\pm$8.38&37.62$\pm$1.41\\
    \cmidrule(r){1-1}\cmidrule(r){2-3}\cmidrule(r){4-5}\cmidrule(r){6-7}\cmidrule(r){8-9}
    L2P~\cite{wang2022learning}&36.08$\pm$0.65&41.66$\pm$0.87&33.55$\pm$0.70&38.95$\pm$1.18&33.64$\pm$0.56&39.08$\pm$0.07&32.39$\pm$0.65&37.55$\pm$0.71\\
    DualPrompt~\cite{wang2022dualprompt}&42.42$\pm$3.12&53.20$\pm$0.83&48.09$\pm$0.14&56.99$\pm$0.01&48.02$\pm$0.33&57.19$\pm$0.01&47.66$\pm$0.84&57.69$\pm$0.47\\
    CODA-Prompt~\cite{smith2023coda}&52.17$\pm$0.23&62.64$\pm$0.39&42.59$\pm$0.65&49.77$\pm$1.08&39.30$\pm$0.37&47.66$\pm$0.32&47.53$\pm$0.28&58.05$\pm$0.21\\
    InfLoRA~\cite{liang2024inflora}&55.47$\pm$1.72&68.04$\pm$1.64&51.61$\pm$0.93&65.89$\pm$0.12&50.56$\pm$3.45&65.67$\pm$1.62&45.30$\pm$3.91&60.28$\pm$0.39\\
    BiLoRA~\cite{zhu2025bilora}&56.45$\pm$0.04&67.54$\pm$0.52&56.32$\pm$0.69&67.25$\pm$0.93&57.24$\pm$1.63&69.19$\pm$0.75&56.12$\pm$0.98&69.55$\pm$1.07\\
    LoRA$^{-}$DRS~\cite{liu2025lora}&59.09$\pm$1.35&70.35$\pm$0.93&58.50$\pm$1.07&69.77$\pm$0.27&58.62$\pm$0.33&69.52$\pm$0.01&56.91$\pm$0.23&68.22$\pm$0.28\\
    \cmidrule(r){1-1}\cmidrule(r){2-3}\cmidrule(r){4-5}\cmidrule(r){6-7}\cmidrule(r){8-9}
    \textbf{PaLoRA (Ours)}&\textbf{61.55$\pm$0.56}&\textbf{72.91$\pm$0.15}&\textbf{61.66$\pm$0.23}&\textbf{72.43$\pm$0.37}&\textbf{62.67$\pm$0.18}&\textbf{73.62$\pm$0.57}&\textbf{60.90$\pm$0.28}&\textbf{72.41$\pm$0.16}\\
    \bottomrule
  \end{tabular}}
\end{table}

\begin{table}[t]
  \centering
  \caption{Backward transfer (BWT, \%; higher is better) on all three datasets under 25- and 50-task splits. PaLoRA yields the least forgetting on CIFAR-100 and ImageNet-A, while LoRA$^{-}$DRS has slightly better BWT on ImageNet-R. Values are point estimates.}
  \label{table:bwt}
  \resizebox{0.75\linewidth}{!}{%
  \begin{tabular}{lrrrrrr}
    \toprule
    Method & \multicolumn{2}{c}{CIFAR-100} & \multicolumn{2}{c}{ImageNet-R} & \multicolumn{2}{c}{ImageNet-A} \\
    \cmidrule(lr){2-3}\cmidrule(lr){4-5}\cmidrule(lr){6-7}
    & 25 Tasks & 50 Tasks & 25 Tasks & 50 Tasks & 25 Tasks & 50 Tasks \\
    \midrule
    Sequence & -16.24 & -14.72 & -36.86 & -42.14 & -31.62 & -53.83 \\
    L2P~\cite{wang2022learning} & -8.94 & -14.07 & -9.43 & -10.56 & -6.80 & -11.16 \\
    DualPrompt~\cite{wang2022dualprompt} & -8.35 & -6.46 & -8.37 & -9.72 & -6.36 & -8.79 \\
    CODA-Prompt~\cite{smith2023coda} & -9.83 & -17.11 & -6.02 & -9.76 & -10.71 & -9.46 \\
    InfLoRA~\cite{liang2024inflora} & -12.53 & -20.45 & -15.08 & -20.01 & -16.47 & -18.34 \\
    BiLoRA~\cite{zhu2025bilora} & -6.83 & -17.48 & -5.36 & -6.58 & -6.66 & -8.08 \\
    LoRA$^{-}$DRS~\cite{liu2025lora} & -4.64 & -5.26 & \textbf{-4.78} & \textbf{-4.92} & -6.06 & -9.01 \\
    \textbf{PaLoRA (Ours)} & \textbf{-3.78} & \textbf{-4.72} & -5.34 & -5.20 & \textbf{-4.87} & \textbf{-7.09} \\
    \bottomrule
  \end{tabular}}
\end{table}
\begin{figure}
  
  \centering
  \resizebox{1.0\linewidth}{!}{%
  \begin{tabular}{c}
  \includegraphics[width=1.0\linewidth]{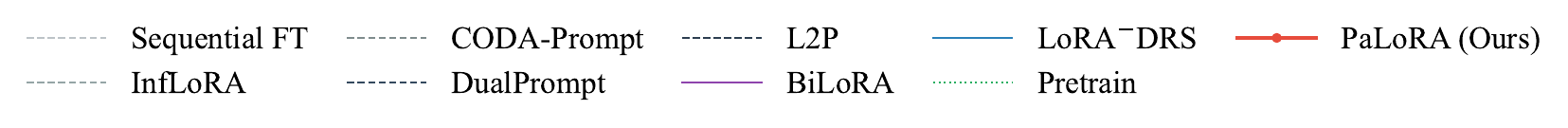}\\
   \includegraphics[width=1.0\linewidth]{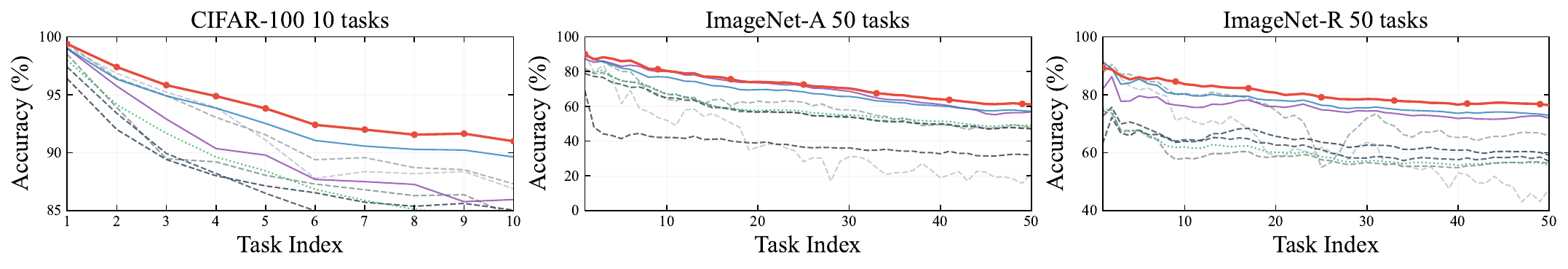}\\
  \end{tabular}}
  \caption{Accuracy trajectories on CIFAR-100 (10 tasks), ImageNet-A (50 tasks), and ImageNet-R (50 tasks), from left to right. PaLoRA maintains strong accuracy, including over long task sequences.}
  \vspace{-6pt}
  \label{fig:3_curve}
\end{figure}

\subsection{Main Results}

\textbf{Overall Performance.}
Tables~\ref{table:imgr}, \ref{table:cifar}, and~\ref{table:imga} report five-seed results on ImageNet-R, CIFAR-100, and ImageNet-A. PaLoRA achieves the highest mean final and average accuracy in all settings. On ImageNet-R, its final-accuracy gain over LoRA$^{-}$DRS ranges from 3.19 to 4.09 percentage points, reaching $76.56\pm0.23\%$ at 50 tasks. On CIFAR-100, the corresponding gains range from 1.47 to 2.32 points. On ImageNet-A, PaLoRA improves final accuracy by 2.46--4.05 points, with a 3.99-point gain at 50 tasks. The five-seed results support the consistency of these improvements, including in long task sequences.

\textbf{Accuracy Evolution.}
Figure~\ref{fig:3_curve} shows accuracy trajectories on CIFAR-100 (10 tasks), ImageNet-A (50 tasks), and ImageNet-R (50 tasks). PaLoRA maintains strong performance through the task sequence, with clear gains in the two long-horizon settings. Full trajectories across datasets and task splits are provided in Appendix~\ref{app:accuracy-curves}.

\textbf{Backward Transfer.}
Table~\ref{table:bwt} consolidates BWT on all three datasets for 25 and 50 tasks. PaLoRA achieves the best BWT on CIFAR-100 and ImageNet-A; at 50 tasks, it improves over LoRA$^{-}$DRS from $-5.26$ to $-4.72$ and from $-9.01$ to $-7.09$, respectively. On ImageNet-R, its BWT is slightly worse ($-5.34/-5.20$ versus $-4.78/-4.92$ for 25/50 tasks), while its final accuracy is higher (Table~\ref{table:imgr}).

\subsection{Comparison with SD-LoRA}
\label{sec:sd-lora}
SD-LoRA~\cite{DBLP:conf/iclr/WuPHW0PMM025} decouples the magnitude and direction of LoRA updates and learns task-specific magnitude scalars. PaLoRA instead determines pacing from the accumulated effective rank, without learning a magnitude scalar for each task. Table~\ref{table:sd-lora} reports their direct comparison. PaLoRA performs better in every reported setting, with larger final-accuracy gaps at 50 tasks than at 10 tasks. At 50 tasks, the gains are 6.72, 21.16, and 29.10 percentage points on ImageNet-R, CIFAR-100, and ImageNet-A, respectively. 

\begin{table}[t]
  \centering
  \caption{Direct comparison with SD-LoRA (accuracy in \%; point estimates). PaLoRA achieves larger final-accuracy gains at 50 tasks than at 10 tasks on all three datasets.}
  \label{table:sd-lora}
  \resizebox{0.8\linewidth}{!}{%
  \begin{tabular}{llrrrrrrrr}
    \toprule
    \multirow{2}{*}{Dataset} & \multirow{2}{*}{Method} & \multicolumn{2}{c}{10 Tasks} & \multicolumn{2}{c}{20 Tasks} & \multicolumn{2}{c}{25 Tasks} & \multicolumn{2}{c}{50 Tasks} \\
    \cmidrule(lr){3-4}\cmidrule(lr){5-6}\cmidrule(lr){7-8}\cmidrule(lr){9-10}
    & & $ACC_{10}$ & $\overline{ACC_{10}}$ & $ACC_{20}$ & $\overline{ACC_{20}}$ & $ACC_{25}$ & $\overline{ACC_{25}}$ & $ACC_{50}$ & $\overline{ACC_{50}}$ \\
    \midrule
    \multirow{2}{*}{ImageNet-R} & SD-LoRA & 78.33 & 82.85 & 74.28 & 79.45 & 75.37 & 80.82 & 69.68 & 76.24 \\
     & \textbf{PaLoRA} & \textbf{78.82} & \textbf{82.96} & \textbf{78.57} & \textbf{82.74} & \textbf{78.10} & \textbf{82.90} & \textbf{76.40} & \textbf{80.40} \\
    \midrule
    \multirow{2}{*}{CIFAR-100} & SD-LoRA & 88.05 & 92.47 & 84.20 & 90.48 & 80.32 & 87.54 & 67.70 & 77.60 \\
     & \textbf{PaLoRA} & \textbf{90.98} & \textbf{93.99} & \textbf{90.31} & \textbf{93.67} & \textbf{90.25} & \textbf{93.69} & \textbf{88.86} & \textbf{92.53} \\
    \midrule
    \multirow{2}{*}{ImageNet-A} & SD-LoRA & 57.21 & 70.68 & 50.63 & 65.06 & 46.94 & 58.19 & 31.99 & 54.78 \\
     & \textbf{PaLoRA} & \textbf{61.95} & \textbf{73.02} & \textbf{61.82} & \textbf{72.70} & \textbf{62.54} & \textbf{73.21} & \textbf{61.09} & \textbf{72.52} \\
    \bottomrule
  \end{tabular}}
\end{table}

\subsection{Ablation Study}
\textbf{Module Ablation.} To better understand the contribution of each component, we conduct ablation experiments on ImageNet-A, where \textbf{Sequence} performs the worst. We decompose our method into three parts: the regularization term, the nullspace gradient projection, and the core pacing mechanism (which includes SVD-based compression of previous knowledge and gradient scaling based on the resulting effective rank $R$).

We report results under three task settings, as shown in Figure~\ref{fig:module_ablation}. Specifically, \textbf{Sequence} includes none of the components, \textbf{proj} uses only nullspace gradient projection, \textbf{no\_reg} combines projection with the pacing mechanism, and PaLoRA includes all three components.

\begin{figure}[b]

  \centering
  \resizebox{1.0\linewidth}{!}{%
  \begin{tabular}{ccc}
  \includegraphics[width=0.30\linewidth]{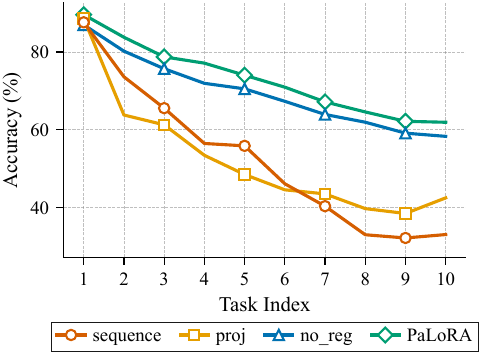}&
  \includegraphics[width=0.30\linewidth]{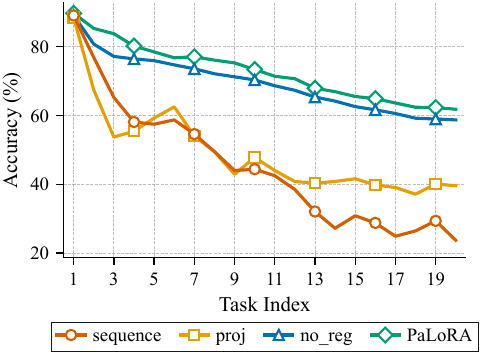}&
  \includegraphics[width=0.30\linewidth]{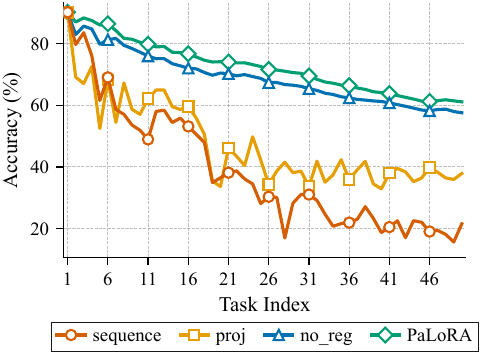}\\
  a) 10 Tasks & b) 20 Tasks& c) 50 Tasks \\
  \end{tabular}}
  \caption{Component ablation on ImageNet-A with 10, 20, and 50 tasks. Projection and pacing improve the trajectories, while triplet regularization provides an additional accuracy gain. Each setting uses a different class split, so early-task curves are not expected to coincide across panels.}
    \label{fig:module_ablation}
\end{figure}

Projection alone mitigates forgetting, while adding rank-aware pacing further improves the accuracy trajectories. The nearly parallel curves with and without regularization suggest an additional plasticity benefit.

Table~\ref{table:triplet-ablation} quantifies the effect of triplet regularization on CIFAR-100. At 50 tasks, PaLoRA without the triplet loss reaches $88.23\%$ final accuracy versus $81.72\%$ for Sequence, and the triplet loss adds a further 0.63 points within PaLoRA. In contrast, adding triplet loss to Sequence reduces final accuracy to $80.00\%$. Together with Figure~\ref{fig:module_ablation}, these results support the pacing framework as the main source of improvement and the triplet loss as an auxiliary plasticity compensator.

\begin{table}[t]
  \centering
  \caption{Triplet-loss ablation on CIFAR-100 (accuracy in \%; point estimates). PaLoRA retains most of its improvement without the triplet loss, which provides a smaller additional gain.}
  \label{table:triplet-ablation}
  \resizebox{0.65\linewidth}{!}{%
  \begin{tabular}{lrrrr}
    \toprule
    Method & $ACC_{10}$ & $\overline{ACC_{10}}$ & $ACC_{50}$ & $\overline{ACC_{50}}$ \\
    \midrule
    Sequence & 86.88 & 91.57 & 81.72 & 87.99 \\
    Sequence + triplet loss & 87.93 & 92.34 & 80.00 & 89.25 \\
    PaLoRA w/o triplet loss & 89.59 & 93.18 & 88.23 & 92.40 \\
    \textbf{PaLoRA} & 90.98 & 93.99 & 88.86 & 92.53 \\
    \bottomrule
  \end{tabular}}
\end{table}

\textbf{Validation of the Pacing Factor Exponent.} We vary the gradient multiplier as $\kappa/R^{\alpha}$, with fixed $\kappa$ and exponent $\alpha$; its corresponding denominator is $s=R^{\alpha}/\kappa$. Thus, Eqn.~\eqref{eq:sqrt-law} predicts $\alpha=0.5$ (with $\kappa=\sqrt{c}$ for the theoretical normalization). We evaluate $\alpha \in \{0.125, 0.25, 0.5, 1.0, 2.0\}$. The results are shown in Figure~\ref{fig:a_alpha}. We report three metrics: \textbf{Accuracy}, which reflects overall performance; \textbf{Training Loss} measured at the end of each task, which indicates model plasticity; and \textbf{Forgetting}, defined as the average performance drop on all learned tasks after each new task, which measures model stability.

The results show that $\alpha = 0.5$ achieves the best performance. When $\alpha > 0.5$, the model suffers from reduced plasticity, leading to performance degradation. In contrast, when $\alpha < 0.5$, the model becomes less stable, resulting in increased forgetting and lower overall performance.

Additional effective-rank comparisons are provided in Appendix~\ref{app:rank-definitions}; sensitivity, computational overhead, LoRA-variant, and backbone results are reported in Appendix~\ref{app:more-experiments}.

\begin{figure}
  
  \centering
  \resizebox{1.0\linewidth}{!}{%
  \begin{tabular}{ccc}
  \includegraphics[width=0.30\linewidth]{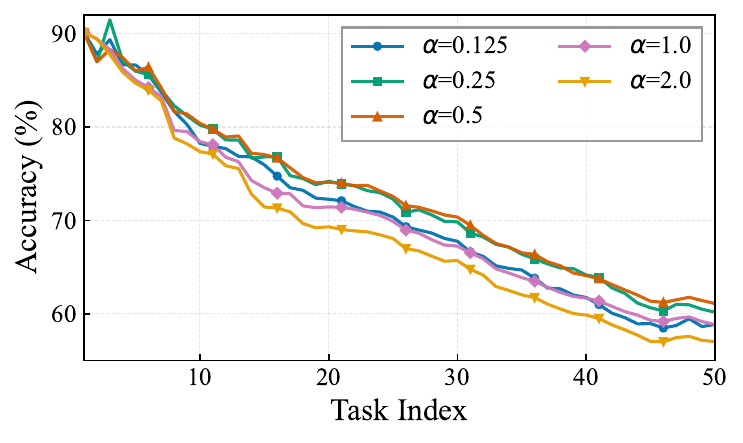}&
  \includegraphics[width=0.30\linewidth]{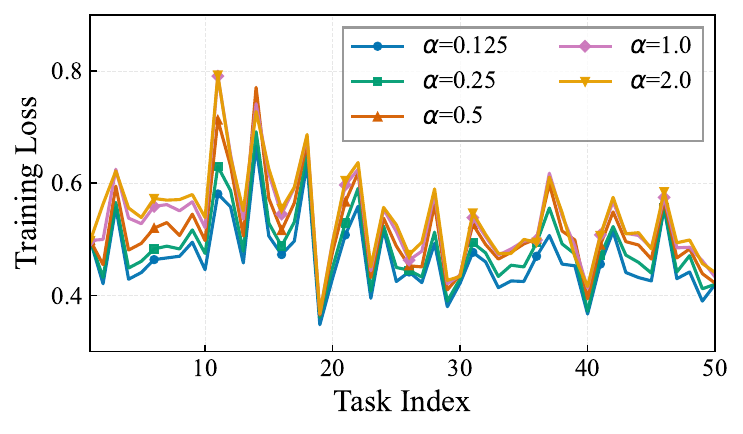}&
  \includegraphics[width=0.30\linewidth]{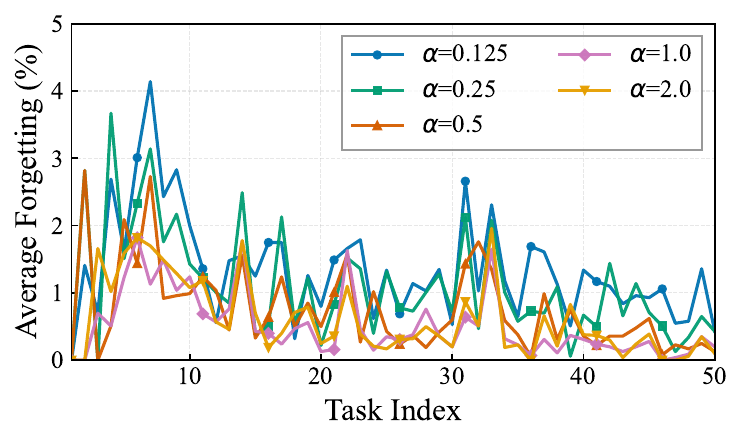}\\
  a) Accuracy & b) Training Loss & c) Forgetting \\
  \end{tabular}}
  \caption{Pacing-exponent ablation on ImageNet-A with 50 tasks. The exponent $\alpha=0.5$ gives the best accuracy while balancing training loss and forgetting, consistent with the predicted square-root denominator scaling.}
  \label{fig:a_alpha}
\end{figure}

\section{Limitations and Future Work}
Our pacing law is derived in a single-layer linear setting under anisotropic leakage and local quadratic sensitivity assumptions. Its application to deep ViTs provides an approximate scaling principle rather than a formal guarantee; the empirical results do not establish that these assumptions hold in the full nonlinear network. The Pareto-optimality result concerns the budget-constrained pacing mechanism, while the auxiliary triplet loss is an empirically motivated plasticity compensator.

Our evaluation is limited to image classification. Extending PaLoRA to generative models, including LLMs and MLLMs, remains future work. The update-pacing principle may also apply beyond LoRA, but suitable measures of accumulated knowledge require further investigation. Although the sensitivity experiments support the chosen pacing constant and SVD threshold, their behavior in broader continual learning settings remains to be studied.
\section{Conclusion}
We show that LoRA-based continual learning suffers from anisotropic leakage into previously acquired subspaces, rendering fixed gradient scaling heuristics inherently suboptimal. Building on this insight, we derive the rank-aware pacing law $s^*=\sqrt{R/c}$ and theoretically establish its Pareto optimality under the stability--plasticity trade-off. Motivated by this principle, we propose PaLoRA, which combines adaptive SVD compression with rank-dependent gradient pacing. Extensive experiments demonstrate that PaLoRA consistently improves performance across benchmarks, achieving up to 4\% gains on the challenging 50-task ImageNet-A and ImageNet-R continual learning settings.
\section*{Acknowledgements}
This work is supported by the Fundamental Research Funds for the Central Universities, Peking University.

\medskip

{
\small
\bibliographystyle{plain}
\bibliography{ref}
}

\clearpage
\appendix
\appendix

\section{Full Theoretical Details of Rank-Aware Scaling}
\label{app:rank-aware-proofs}

\subsection{Setup and Notation}
\label{app:proofs-setup}

Consider a single linear layer with frozen pre-trained weights $W_0\in\RR^{d\times d}$. In a continual learning setting, after learning tasks $1,\dots,t-1$ via LoRA, the accumulated low-rank adaptations form an effective update $C_{t-1}\in\RR^{d\times d}$ with $\rank(C_{t-1})=R$, so that the current effective weight is:
\[
W^{(t-1)}=W_0+C_{t-1}.
\]
Task $t$ introduces an additional LoRA increment $\Delta_t\in\RR^{d\times d}$, and the updated weight becomes:
\[
W^{(t)}=W^{(t-1)}+\Delta_t=W_0+C_{t-1}+\Delta_t.
\]
Each task $t$ is equipped with an input design matrix $X_t\in\RR^{d\times n_t}$ with positive-definite covariance $\Sigma_t=X_tX_t^\top\succ 0$. The task-$t$ objective is the Frobenius-norm regression loss:
\[
\cL_t(\Delta_t)=\frac{1}{2}\|(W^{(t-1)}+\Delta_t)X_t-Y_t\|_F^2,
\]
optimized via full-batch gradient descent with fixed-step projected gradient descent.

After completing task $t-1$, compute the compact SVD of the accumulated LoRA adaptation:
\begin{equation}
\label{eq:svd-cumulative}
C_{t-1}=U\Sigma V^\top=\sum_{i=1}^{R}\sigma_i u_i v_i^\top,\qquad R=\rank(C_{t-1}).
\end{equation}
Let $V_\perp=[v_{R+1},\dots,v_d]$ span the right nullspace (orthogonal complement of the row space) of $C_{t-1}$. We constrain the task-$t$ LoRA increment to the feasible subspace:
\[
\cN_t\triangleq\{M V_\perp^\top : M\in\RR^{d\times(d-R)}\}.
\]
Equivalently, $\Delta_t V_{[:R]}=0$, which implies $\Delta_t x_{\old}=0$ for all old-task inputs $x_{\old}\in\Span(V_{[:R]})$. Under exact arithmetic, old-task forward passes are therefore preserved.

Let $P_{\cN_t}$ denote the Frobenius-orthogonal projection onto $\cN_t$, which acts as $P_{\cN_t}(G)=G V_\perp V_\perp^\top$ for any $G\in\RR^{d\times d}$. The ideal task-$t$ update is:
\[
\delta\Delta_t=-\frac{\eta}{s}P_{\cN_t}(\nabla_{\Delta_t}\cL_t),
\]
where $\eta$ is the base learning rate, $s>0$ is the step-size denominator, and $\nabla_{\Delta_t}\cL_t=((W^{(t-1)}+\Delta_t)X_t-Y_t)X_t^\top$ is evaluated at the current iterate.

\subsection{Leakage and Sensitivity Assumptions}
\label{app:leakage-assumptions}

In practice, finite-precision arithmetic and the imperfect implementation of the projector cause the effective update to leak into the old-task subspace $\cS=\Span\{u_i v_i^\top\}_{i=1}^R$.

\begin{assumption}[Anisotropic Leakage]
\label{ass:leakage}
After a scaled projected gradient step, decompose the residual component of the effective weight update in the SVD basis as:
\[
P_{\cS}(\delta\Delta_t)=\sum_{i=1}^{R}\epsilon_i\,u_i v_i^\top.
\]
The per-direction leakage satisfies:
\begin{equation}
\label{eq:anisotropic-leakage}
\EE[\epsilon_i^2]=\frac{\eta^2\gamma_i^2}{s^2},\qquad\forall i\in\{1,\dots,R\},
\end{equation}
where $\gamma_i>0$ captures the direction-dependent base leakage variance. The leakage is independent across directions.
\end{assumption}

\begin{assumption}[Direction-Dependent Quadratic Sensitivity]
\label{ass:quadratic}
Near the optimum $W^{(t-1)}$, the old-task loss is locally quadratic with curvature $\lambda_i>0$ along the $i$-th singular direction:
\[
\cL_{t-1}(W^{(t-1)}+\delta W)-\cL_{t-1}(W^{(t-1)})\approx\frac{1}{2}\sum_{i=1}^{R}\lambda_i\,\langle u_i v_i^\top,\delta W\rangle_F^2.
\]
\end{assumption}

These assumptions characterize residual leakage along historical directions, rather than the full displacement between task optima. For deep ViTs with LoRA, small projected leakage motivates a local quadratic approximation, but its validity is not established by this analysis.

\subsection{The Forgetting--Progress Trade-off}
\label{app:forgetting-progress-tradeoff}

Under Assumptions~\ref{ass:leakage}--\ref{ass:quadratic}, the per-step expected forgetting and first-order new-task progress are given by Eqns.~\eqref{eq:expected-forgetting} and~\eqref{eq:new-task-progress}.

Eqn.~\eqref{eq:expected-forgetting} shows that scaling down the step by $s$ reduces forgetting \emph{quadratically}: halving the step cuts forgetting by a factor of four. Meanwhile, Eqn.~\eqref{eq:new-task-progress} shows that new-task progress degrades only \emph{linearly}: halving the step merely halves the progress. This asymmetric sensitivity is the fundamental mechanism behind the $\sqrt{R/c}$ scaling law: we can pay a small linear price in learning speed for a large quadratic reduction in catastrophic forgetting.

Proofs of Lemmas~\ref{lem:forgetting} and~\ref{lem:progress} are given below in Section~\ref{app:proofs-lemmas}.

\subsection{Proofs of Main Lemmas}
\label{app:proofs-lemmas}

\begin{lemma}[Expected Forgetting per Step]
\label{lem:forgetting}
Under Assumptions~\ref{ass:leakage}--\ref{ass:quadratic}, the expected old-task loss increase from one scaled step is:
\begin{equation}
\label{eq:expected-forgetting-app}
\EE[\Delta\cL_{t-1}]=\frac{\eta^2 R\Gamma^2}{2s^2},
\qquad\text{where}\quad
\Gamma^2\triangleq\frac{1}{R}\sum_{i=1}^{R}\lambda_i\gamma_i^2.
\end{equation}
\end{lemma}
\begin{proof}
Assumption~\ref{ass:quadratic} together with the orthonormality $\langle u_i v_i^\top,u_j v_j^\top\rangle_F=\delta_{ij}$ yields
\[
\Delta\cL_{t-1}\approx\frac{1}{2}\sum_{i=1}^{R}\lambda_i\epsilon_i^2.
\]
Taking expectation and applying Assumption~\ref{ass:leakage}, we obtain:
\[
\EE[\Delta\cL_{t-1}]=\frac{1}{2}\sum_{i=1}^{R}\lambda_i\EE[\epsilon_i^2]
=\frac{1}{2}\sum_{i=1}^{R}\lambda_i\frac{\eta^2\gamma_i^2}{s^2}
=\frac{\eta^2}{2s^2}\sum_{i=1}^{R}\lambda_i\gamma_i^2
=\frac{\eta^2 R\Gamma^2}{2s^2}.
\qedhere
\]
\end{proof}

\begin{lemma}[New-Task Progress per Step]
\label{lem:progress}
For the projected gradient $g=P_{\cN_t}(\nabla_{\Delta_t}\cL_t)$ with squared norm $G^2=\|g\|_F^2$, a step $\delta\Delta_t=-\frac{\eta}{s}g$ yields first-order progress:
\begin{equation}
\label{eq:new-task-progress-app}
\Delta\cL_t\approx-\frac{\eta G^2}{s}.
\end{equation}
Hence the magnitude of new-task progress is $\eta G^2/s$.
\end{lemma}
\begin{proof}
By first-order Taylor expansion of $\cL_t$ around the current iterate:
\[
\Delta\cL_t\approx\langle\nabla_{\Delta_t}\cL_t,\delta\Delta_t\rangle_F
=\Big\langle g,-\frac{\eta}{s}g\Big\rangle_F
=-\frac{\eta}{s}\|g\|_F^2
=-\frac{\eta G^2}{s}.
\qedhere
\]
\end{proof}

\begin{lemma}[Optimal Scaling Factor for a Given Budget]
\label{lem:optimal}
Consider the constrained optimization problem that maximizes new-task progress subject to a per-step forgetting budget $\epsilon_0>0$:
\[
\max_{s>0}\;\frac{\eta G^2}{s}\quad\text{subject to}\quad\frac{\eta^2 R\Gamma^2}{2s^2}\le\epsilon_0.
\]
The unique optimal scale is:
\begin{equation}
\label{eq:optimal-scale-budget}
s^*=\eta\Gamma\sqrt{\frac{R}{2\epsilon_0}}.
\end{equation}
\end{lemma}
\begin{proof}
The objective $\eta G^2/s$ is strictly decreasing in $s$. The constraint yields the feasible set $s\ge\eta\Gamma\sqrt{\frac{R}{2\epsilon_0}}$. Because the objective decreases with $s$, the optimum is attained at the left boundary, giving $s^*=\eta\Gamma\sqrt{\frac{R}{2\epsilon_0}}$.
\end{proof}

\begin{definition}[Average Per-Direction Noise Floor]
\label{def:floor}
The average per-direction noise floor $\epsilon_{\dir}$ is the mean expected old-task forgetting from one unscaled step ($s=1$) leaking into a single singular direction:
\begin{equation}
\label{eq:noise-floor-app}
\epsilon_{\dir}\triangleq\frac{1}{R}\sum_{i=1}^{R}\frac{\lambda_i\eta^2\gamma_i^2}{2}
=\frac{\eta^2\Gamma^2}{2}.
\end{equation}
\end{definition}

We set the budget to $c$ times this average floor:
\begin{equation}
\label{eq:budget-c-app}
\epsilon_0=c\,\epsilon_{\dir}=\frac{c\,\eta^2\Gamma^2}{2}.
\end{equation}

\begin{lemma}[The $\sqrt{R/c}$ Scaling Law]
\label{lem:sqrt-law-app}
With the natural budget $\epsilon_0=\frac{c\,\eta^2\Gamma^2}{2}$, the optimal scale simplifies to:
\begin{equation}
\label{eq:sqrt-law-app}
s^*=\sqrt{\frac{R}{c}}.
\end{equation}
\end{lemma}
\begin{proof}
Substitute $\epsilon_0=\frac{c\,\eta^2\Gamma^2}{2}$ from Eqn.~\eqref{eq:budget-c-app} into Eqn.~\eqref{eq:optimal-scale-budget}:
\[
s^*=\eta\Gamma\sqrt{\frac{R}{2\cdot\frac{c\,\eta^2\Gamma^2}{2}}}
=\eta\Gamma\sqrt{\frac{R}{c\,\eta^2\Gamma^2}}
=\sqrt{\frac{\eta^2\Gamma^2\cdot R}{c\,\eta^2\Gamma^2}}
=\sqrt{\frac{R}{c}}.
\qedhere
\]
\end{proof}

\subsection{Properties of the Scaling Law}
\label{app:proofs-properties}

Substituting $s^*=\sqrt{R/c}$ into Lemma~\ref{lem:forgetting} gives:
\[
\EE[\Delta\cL_{t-1}]
=\frac{\eta^2 R\Gamma^2}{2(\sqrt{R/c})^2}
=\frac{\eta^2 R\Gamma^2}{2\cdot R/c}
=\frac{c\,\eta^2\Gamma^2}{2}
=\epsilon_0.
\]
Thus the per-step forgetting exactly equals the prescribed budget $c\,\epsilon_{\dir}$, regardless of the specific directions that constitute the current rank-$R$ subspace.

Any $s<\sqrt{R/c}$ violates the budget because $\EE[\Delta\cL_{t-1}]>\epsilon_0$. Any $s>\sqrt{R/c}$ is suboptimal because, by Lemma~\ref{lem:progress}, new-task progress $\eta G^2/s$ is strictly smaller than achievable. Hence $s^*$ is the unique Pareto-optimal operating point for this budget-constrained pacing problem under the stated model, rather than for the full PaLoRA system with auxiliary regularization.
\section{Additional Experimental Results}
\label{app:more-experiments}
\subsection{Implementation Details}
\label{app:implementation}
Our method is implemented in PyTorch. Following prior work~\cite{liang2024inflora,liu2025lora}, we insert LoRA modules into the key and value projections of the attention layers in ViT. Across all datasets and task settings, we fix the SVD compression threshold $\tau$ to 0.99 and fix $C$ to 4 in the pacing factor. The regularization term is weighted by 0.5 with a margin of 2.0.

The main accuracy tables report means and standard deviations over seeds 0--4; the additional analyses report point estimates unless otherwise stated. Experiments are conducted on a single NVIDIA GeForce RTX 4080 SUPER GPU. Unless otherwise specified, we use a ViT-B/16-IN21K pretrained backbone. Training is performed using the Adam optimizer with a learning rate of 0.0005 and $(\beta_1, \beta_2) = (0.9, 0.999)$. We train for 50 epochs on ImageNet-R and ImageNet-A, and 20 epochs on CIFAR-100, with a batch size of 128.

\subsection{Additional Accuracy Curves}
\label{app:accuracy-curves}
The main text shows representative trajectories from all three datasets. Figure~\ref{fig:all_curve} provides the full results across datasets and task splits, showing that PaLoRA maintains strong performance throughout the learning sequences.

\begin{figure}
  \centering
  \resizebox{1.0\linewidth}{!}{%
  \begin{tabular}{c}
  \includegraphics[width=0.30\linewidth]{sec/figures/figure_legend.pdf}\\
   \includegraphics[width=0.30\linewidth]{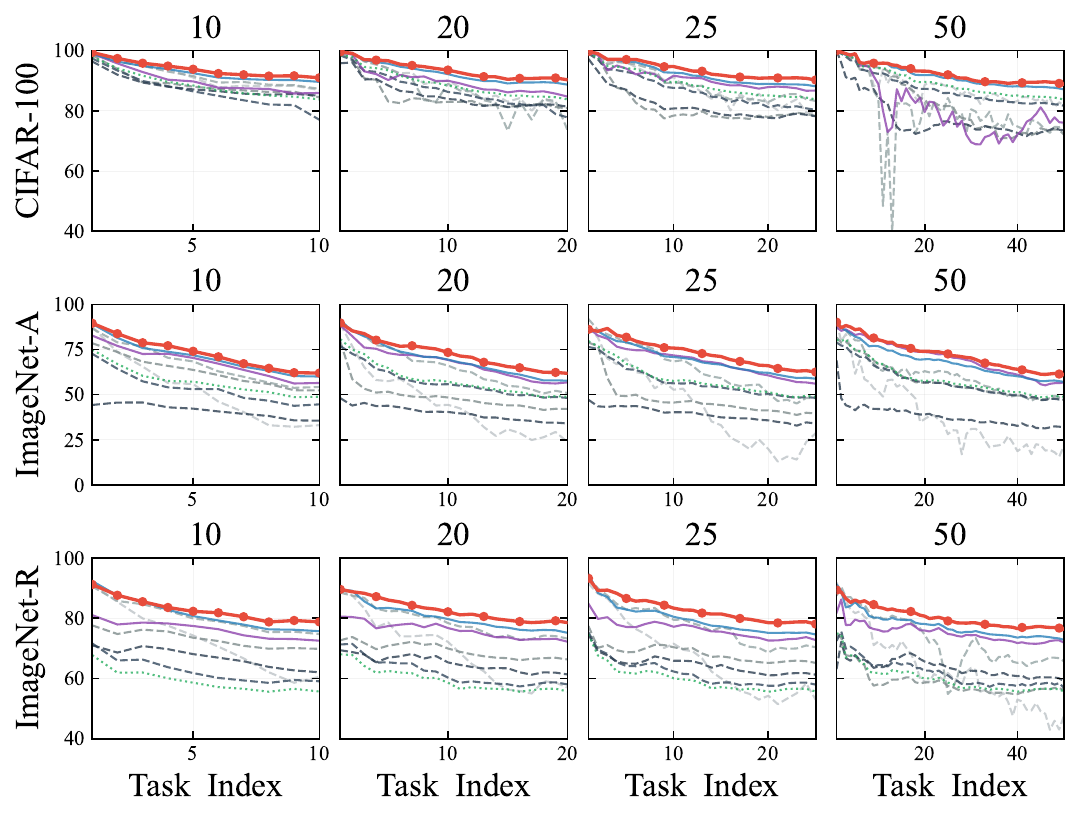}\\
  \end{tabular}}
  \vspace{-14pt}
  \caption{Accuracy trajectories across CIFAR-100, ImageNet-A, and ImageNet-R under 10-, 20-, 25-, and 50-task splits. PaLoRA maintains strong performance across datasets and task horizons.}
  \vspace{-3pt}
  \label{fig:all_curve}
\end{figure}

\subsection{Effectiveness of SVD Compression} We evaluate the effectiveness of the SVD compression module across all three datasets under different task settings. The results are illustrated in Figure~\ref{fig:rank}.

The \textbf{Constant Rank} curve denotes a predefined rank budget, where $R = r \times$ (number of tasks). The other three curves represent the actual effective rank of the accumulated knowledge after each task on the three datasets. From the results, it is clear that under a fixed threshold $\tau = 0.99$, our method achieves adaptive rank allocation according to task difficulty. Instead of growing linearly with the number of tasks, the effective rank adjusts dynamically, demonstrating the efficiency of the proposed compression mechanism.
\begin{figure}
  \centering
  \resizebox{1.0\linewidth}{!}{%
  \begin{tabular}{ccc}
  \includegraphics[width=0.30\linewidth]{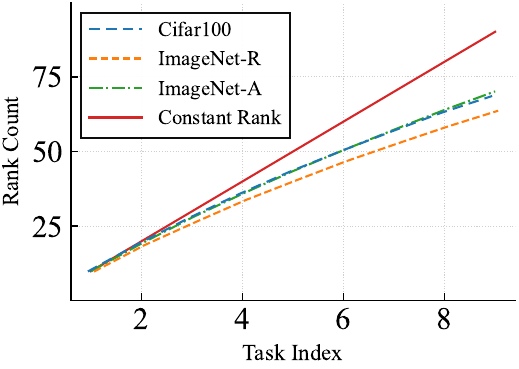}&
  \includegraphics[width=0.30\linewidth]{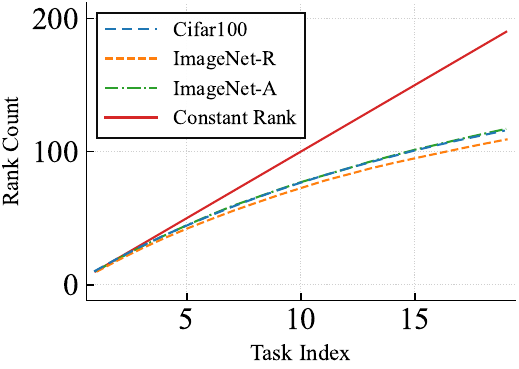}&
  \includegraphics[width=0.30\linewidth]{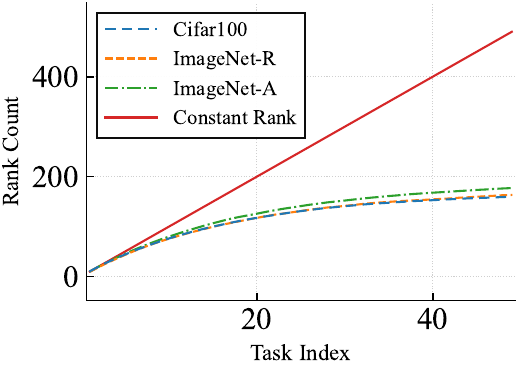}\\
  a) 10 Tasks & b) 20 Tasks& c) 50 Tasks \\
  \end{tabular}}
  \vspace{-2pt}
  \caption{SVD compression across datasets with 10, 20, and 50 tasks. At $\tau=0.99$, the effective rank grows more slowly than the uncompressed $R=rt$ reference and varies with the task sequence.}
  \vspace{-12pt}
  \label{fig:rank}
\end{figure}

\paragraph{Different Effective Rank Definitions.}
\label{app:rank-definitions} We further analyze the impact of different ways of computing the effective rank $R$. As a comparison, we adopt a fixed-rank SVD strategy that retains the top $t \times k$ components, where $t$ is the number of learned tasks and $k$ is a predefined constant. When $k=1$, this setting corresponds to a lower bound, while $k=r$ represents an upper bound (i.e., no compression). As shown in Figure~\ref{fig:a_k}, our method demonstrates strong adaptivity by achieving accuracy curves that are consistently close to the upper bound, while still performing compression according to task difficulty. These results support using adaptive effective rank to balance compression and accuracy.

\begin{figure}
  
  \centering
  \resizebox{0.8\linewidth}{!}{%
  \begin{tabular}{ccc}
    \multicolumn{3}{c}{\includegraphics[width=0.9\linewidth]{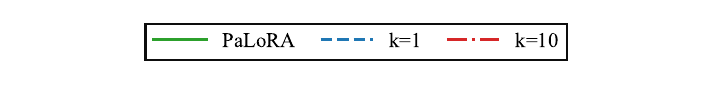}} \\
    \includegraphics[width=0.32\linewidth]{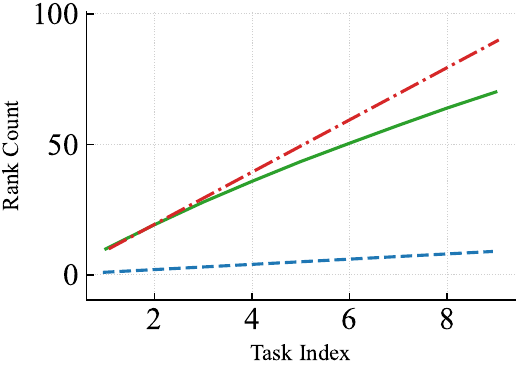} &
    \includegraphics[width=0.32\linewidth]{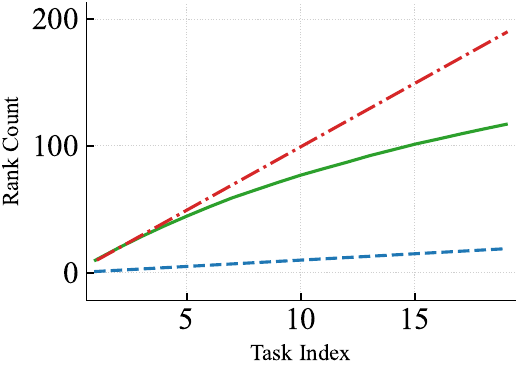} &
    \includegraphics[width=0.32\linewidth]{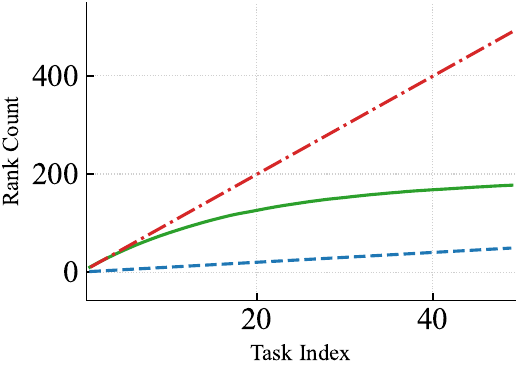} \\
    \includegraphics[width=0.32\linewidth]{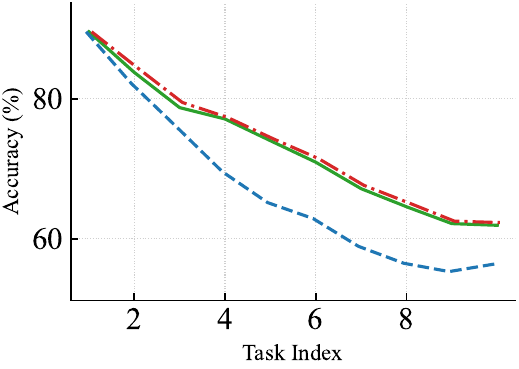} &
    \includegraphics[width=0.32\linewidth]{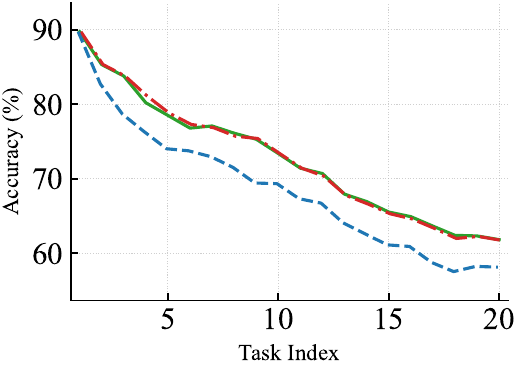} &
    \includegraphics[width=0.32\linewidth]{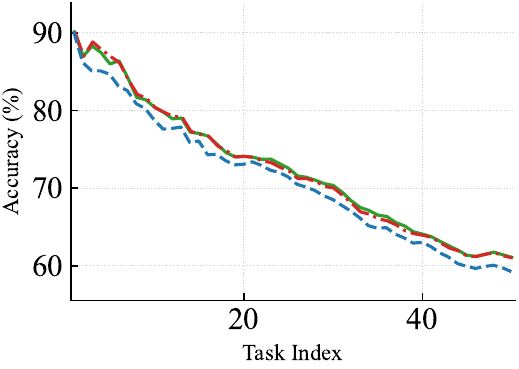} \\
  \end{tabular}}
  \caption{Effective-rank choices on ImageNet-A with 10, 20, and 50 tasks. Top: rank growth; bottom: accuracy. Adaptive SVD retains fewer directions than the $k=r$ reference while maintaining comparable accuracy.}
  \label{fig:a_k}
\end{figure}

\subsection{Hyperparameter Sensitivity}
\label{app:sensitivity}
We vary the pacing constant $C$ and SVD threshold $\tau$ on CIFAR-100 with 50 tasks. With $\tau=0.99$, changing $C$ from 2.25 to 6.25 yields final accuracy between $88.70\%$ and $88.92\%$ (Table~\ref{table:sensitivity-c}). With $C=4$, lowering $\tau$ to 0.95 reduces the average final rank to 39.6 but lowers final accuracy to $84.16\%$. Increasing $\tau$ from 0.99 to 0.999 raises the rank from 159.8 to 390.8 with little change in accuracy (Table~\ref{table:sensitivity-tau}). These results support the default $C=4$ and $\tau=0.99$ as a balance between compression and performance.

\begin{table}[t]
  \centering
  \caption{Pacing-constant sensitivity on CIFAR-100 with 50 tasks and $\tau=0.99$ (accuracy in \%; point estimates). Performance varies little over the tested range.}
  \label{table:sensitivity-c}
  \resizebox{0.45\linewidth}{!}{%
  \begin{tabular}{lrr}
    \toprule
    $C$ & $ACC_{50}$ & $\overline{ACC_{50}}$ \\
    \midrule
    2.25 & 88.92 & 92.54 \\
    4.00 (default) & 88.86 & 92.53 \\
    6.25 & 88.70 & 92.40 \\
    \bottomrule
  \end{tabular}}
\end{table}

\begin{table}[h]
  \centering
  \caption{SVD-threshold sensitivity on CIFAR-100 with 50 tasks and $C=4$ (accuracy in \%; point estimates). The default threshold balances retained rank and accuracy.}
  \label{table:sensitivity-tau}
  \resizebox{0.55\linewidth}{!}{%
  \begin{tabular}{lccc}
    \toprule
    $\tau$ & $ACC_{50}$ & $\overline{ACC_{50}}$ & Avg. Final Rank \\
    \midrule
    0.95 & 84.16 & 90.56 & 39.6 \\
    0.99 (default) & 88.86 & 92.53 & 159.8 \\
    0.999 & 88.74 & 92.55 & 390.8 \\
    \bottomrule
  \end{tabular}}
\end{table}

\subsection{Computational Overhead}
\label{app:overhead}
SVD compression is performed once per task boundary and takes approximately 0.08 seconds per task on an NVIDIA GeForce RTX 4080 SUPER, totaling about 4 seconds over 50 tasks. The triplet loss computes pairwise feature distances with complexity $O(B^2d)$, where $B=128$ is the batch size and $d$ is the feature dimension. Table~\ref{table:runtime} reports average training time per task on CIFAR-100 with 50 tasks. PaLoRA takes 3.01 minutes versus 2.67 for Sequence and 4.36 for LoRA$^{-}$DRS. Without triplet loss, PaLoRA takes 2.81 minutes, quantifying the cost of compression, projection, and pacing together; the triplet loss adds a further 0.20 minutes within PaLoRA.

\begin{table}[t]
  \centering
  \caption{Average training time per task on CIFAR-100 with 50 tasks using an RTX 4080 SUPER. PaLoRA adds 0.34 minutes over Sequence and is faster than LoRA$^{-}$DRS.}
  \label{table:runtime}
  \resizebox{0.55\linewidth}{!}{%
  \begin{tabular}{lc}
    \toprule
    Method & Time per Task (min) \\
    \midrule
    L2P & 2.98 \\
    DualPrompt & 2.95 \\
    CODA-Prompt & 3.02 \\
    BiLoRA & 3.45 \\
    Sequence & 2.67 \\
    Sequence + triplet loss & 2.80 \\
    InfLoRA & 3.31 \\
    LoRA$^{-}$DRS & 4.36 \\
    PaLoRA w/o triplet loss & 2.81 \\
    PaLoRA & 3.01 \\
    \bottomrule
  \end{tabular}}
\end{table}

\subsection{Compatibility with Other LoRA Variants}
\label{app:pissa}
The pacing formulation uses the aggregated low-rank task update rather than a specific factor initialization. We evaluate its compatibility with PiSSA~\cite{DBLP:conf/nips/MengWZ24} on CIFAR-100 with 50 tasks (Table~\ref{table:pissa}). Applying PaLoRA raises the final accuracy of sequential PiSSA adaptation from $71.85\%$ to $89.30\%$, slightly exceeding the $88.86\%$ obtained by PaLoRA with standard LoRA in this comparison.

\begin{table}[h]
  \centering
  \caption{Compatibility with PiSSA on CIFAR-100 with 50 tasks (accuracy in \%; point estimates). PaLoRA improves sequential PiSSA adaptation and slightly exceeds its standard-LoRA counterpart.}
  \label{table:pissa}
  \resizebox{0.5\linewidth}{!}{%
  \begin{tabular}{lrr}
    \toprule
    Method & $ACC_{50}$ & $\overline{ACC_{50}}$ \\
    \midrule
    Sequence & 81.72 & 87.99 \\
    PaLoRA & 88.86 & 92.53 \\
    Sequence + PiSSA & 71.85 & 80.54 \\
    PaLoRA + PiSSA & 89.30 & 92.75 \\
    \bottomrule
  \end{tabular}}
\end{table}

\subsection{Different Pretrained Backbones}
Table~\ref{table:imagenet} compares DINOv1 and iBOT backbones on ImageNet-R and ImageNet-A with 50 tasks. PaLoRA achieves the highest final and average accuracy in all four settings, extending the comparison beyond the default ViT-B/16-IN21K backbone.

\begin{table}[h]
  \caption{Final and average accuracy (\%; point estimates) with DINOv1 and iBOT backbones on 50-task ImageNet-R and ImageNet-A. PaLoRA outperforms the compared methods for both backbones.}
  \label{table:imagenet}
  \centering
  \resizebox{0.85\linewidth}{!}{
  \begin{tabular}{lcccccccc}
    \toprule
    & \multicolumn{4}{c}{ImageNet-R (50 tasks)} & \multicolumn{4}{c}{ImageNet-A (50 tasks)} \\
    \cmidrule(r){2-5}\cmidrule(r){6-9}
    Method & \multicolumn{2}{c}{DINOv1} & \multicolumn{2}{c}{iBOT} & \multicolumn{2}{c}{DINOv1} & \multicolumn{2}{c}{iBOT} \\
    \cmidrule(r){2-3}\cmidrule(r){4-5}\cmidrule(r){6-7}\cmidrule(r){8-9}
    & ACC & $\overline{ACC}$ & ACC & $\overline{ACC}$ & ACC & $\overline{ACC}$ & ACC & $\overline{ACC}$ \\
    \midrule
    InfLoRA~\cite{liang2024inflora} 
    & 61.52 & 70.10 & 58.52 & 69.32 & 30.55 & 42.48 & 33.31 & 49.40 \\
    BiLoRA~\cite{zhu2025bilora} 
    & 62.85 & 66.73 & 65.85 & 70.42 & 28.57 & 37.82 & 34.69 & 44.50 \\
    LoRA$^{-}$DRS~\cite{liu2025lora} 
    & 61.58 & 67.26 & 65.12 & 70.94 & 30.28 & 40.19 & 34.89 & 45.37 \\
    \textbf{PaLoRA (Ours)} 
    & \textbf{66.78} & \textbf{72.80} & \textbf{69.85} & \textbf{75.66} & \textbf{35.94} & \textbf{46.51} & \textbf{43.25} & \textbf{55.18} \\
    \bottomrule
  \end{tabular}}
\end{table}


\newpage
\section*{NeurIPS Paper Checklist}

\begin{enumerate}

\item {\bf Claims}
    \item[] Question: Do the main claims made in the abstract and introduction accurately reflect the paper's contributions and scope?
    \item[] Answer: \answerYes{} 
    \item[] Justification: The abstract and introduction accurately reflect the paper’s contribution and scope.

\item {\bf Limitations}
    \item[] Question: Does the paper discuss the limitations of the work performed by the authors?
    \item[] Answer: \answerYes{} 
    \item[] Justification: We discuss the limitations in the main-text section ``Limitations and Future Work''.

\item {\bf Theory assumptions and proofs}
    \item[] Question: For each theoretical result, does the paper provide the full set of assumptions and a complete (and correct) proof?
    \item[] Answer: \answerYes{} 
    \item[] Justification: We provide the full set of assumptions and a complete proof in the Appendix.

    \item {\bf Experimental result reproducibility}
    \item[] Question: Does the paper fully disclose all the information needed to reproduce the main experimental results of the paper to the extent that it affects the main claims and/or conclusions of the paper (regardless of whether the code and data are provided or not)?
    \item[] Answer: \answerYes{} 
    \item[] Justification: We provide the code and logs in the supplements and we provide implement details in Appendix.

\item {\bf Open access to data and code}
    \item[] Question: Does the paper provide open access to the data and code, with sufficient instructions to faithfully reproduce the main experimental results, as described in supplemental material?
    \item[] Answer: \answerYes{} 
    \item[] Justification: We provide the code and logs in the supplements.

\item {\bf Experimental setting/details}
    \item[] Question: Does the paper specify all the training and test details (e.g., data splits, hyperparameters, how they were chosen, type of optimizer) necessary to understand the results?
    \item[] Answer: \answerYes{} 
    \item[] Justification: We provide implement details in Appendix.

\item {\bf Experiment statistical significance}
    \item[] Question: Does the paper report error bars suitably and correctly defined or other appropriate information about the statistical significance of the experiments?
    \item[] Answer: \answerYes{} 
    \item[] Justification: The main accuracy tables report means and standard deviations over five random seeds (0--4). Other tables and the existing accuracy curves report point estimates, as stated in the experimental setup.

\item {\bf Experiments compute resources}
    \item[] Question: For each experiment, does the paper provide sufficient information on the computer resources (type of compute workers, memory, time of execution) needed to reproduce the experiments?
    \item[] Answer: \answerYes{} 
    \item[] Justification: We provide implement details (include compute resources) in Appendix.
    
\item {\bf Code of ethics}
    \item[] Question: Does the research conducted in the paper conform, in every respect, with the NeurIPS Code of Ethics \url{https://neurips.cc/public/EthicsGuidelines}?
    \item[] Answer: \answerYes{} 
    \item[] Justification: The research conducted in our paper conform, in every respect, with the NeurIPS Code of Ethics.

\item {\bf Broader impacts}
    \item[] Question: Does the paper discuss both potential positive societal impacts and negative societal impacts of the work performed?
    \item[] Answer: \answerNA{} 
    \item[] Justification: The paper has no societal impact.
    
\item {\bf Safeguards}
    \item[] Question: Does the paper describe safeguards that have been put in place for responsible release of data or models that have a high risk for misuse (e.g., pre-trained language models, image generators, or scraped datasets)?
    \item[] Answer: \answerNA{} 
    \item[] Justification: The paper poses no such risks.

\item {\bf Licenses for existing assets}
    \item[] Question: Are the creators or original owners of assets (e.g., code, data, models), used in the paper, properly credited and are the license and terms of use explicitly mentioned and properly respected?
    \item[] Answer: \answerYes{} 
    \item[] Justification: The creators or original owners of assets (e.g., code, data, models) used in the paper are properly credited.

\item {\bf New assets}
    \item[] Question: Are new assets introduced in the paper well documented and is the documentation provided alongside the assets?
    \item[] Answer: \answerNA{} 
    \item[] Justification: The paper does not release new assets.

\item {\bf Crowdsourcing and research with human subjects}
    \item[] Question: For crowdsourcing experiments and research with human subjects, does the paper include the full text of instructions given to participants and screenshots, if applicable, as well as details about compensation (if any)? 
    \item[] Answer: \answerNA{} 
    \item[] Justification: The paper does not involve crowdsourcing nor research with human subjects.

\item {\bf Institutional review board (IRB) approvals or equivalent for research with human subjects}
    \item[] Question: Does the paper describe potential risks incurred by study participants, whether such risks were disclosed to the subjects, and whether Institutional Review Board (IRB) approvals (or an equivalent approval/review based on the requirements of your country or institution) were obtained?
    \item[] Answer: \answerNA{} 
    \item[] Justification: The paper does not involve crowdsourcing nor research with human subjects.

\item {\bf Declaration of LLM usage}
    \item[] Question: Does the paper describe the usage of LLMs if it is an important, original, or non-standard component of the core methods in this research? Note that if the LLM is used only for writing, editing, or formatting purposes and does \emph{not} impact the core methodology, scientific rigor, or originality of the research, declaration is not required.
    \item[] Answer: \answerNA{} 
    \item[] Justification: The core method development in this research does not involve LLMs as any important, original, or non-standard components.

\end{enumerate}

\end{document}